\documentclass[twoside,11pt]{article}

\usepackage{jmlr2e}

\usepackage{times}
\usepackage{graphicx}
\usepackage{subcaption}
\usepackage{amsfonts}
\usepackage{amsmath}
\usepackage{mathtools}
\usepackage{bbm}
\usepackage{color}

\newcommand{\R}{\mathbb{R}}

\newcommand{\E}{\mathbf{E}}
\newtheorem{assumptions}{Assumptions}

\usepackage{placeins}

\DeclareMathOperator{\argmin}{argmin}

\usepackage{enumerate}

\ShortHeadings{Graph-Dependent Implicit Regularisation for Distributed SGD}{Richards and Rebeschini}
\firstpageno{1}

\begin{document}

\title{Graph-Dependent Implicit Regularisation\\for Distributed Stochastic Subgradient Descent}

\author{\name Dominic Richards \email dominic.richards@spc.ox.ac.uk \\
       \addr Department of Statistics\\
       University of Oxford\\
       24-29 St Giles', Oxford, OX1 3LB
       \AND
       \name Patrick Rebeschini \email patrick.rebeschini@stats.ox.ac.uk \\
       \addr Department of Statistics\\
       University of Oxford\\
       24-29 St Giles', Oxford, OX1 3LB
}

\editor{}

\maketitle

\begin{abstract}%
We propose graph-dependent implicit regularisation strategies for distributed stochastic subgradient descent (Distributed SGD) for convex problems in multi-agent learning. Under the standard assumptions of convexity, Lipschitz continuity, and smoothness, we establish statistical learning rates that retain, up to logarithmic terms,  
centralised  statistical guarantees through implicit regularisation (step size tuning and early stopping) with appropriate dependence on the graph topology. Our approach avoids the need for explicit regularisation in  decentralised learning problems, such as adding constraints to the empirical risk minimisation rule. Particularly for distributed methods, the use of implicit regularisation allows the algorithm to remain simple, without projections or dual methods. To prove our results, we establish graph-independent generalisation bounds for Distributed SGD that match the centralised setting (using algorithmic stability), and we establish graph-dependent optimisation bounds that are of independent interest.
We present numerical experiments to show that the qualitative nature of the upper bounds we derive can be representative of real behaviours.

%
%

\end{abstract}


\section{Introduction} 

In machine learning, a canonical  setting
involves assuming that training data is made of independent samples from a certain unknown distribution, and the goal is to construct a model that can perform well on new unseen data from the same distribution \citep{Vapnik:1995:NSL:211359}. Given a certain loss function that measures the performance of a model against an individual data point, the classical framework of regularised empirical risk minimisation involves looking for the model that minimises the empirical risk, i.e., the average loss over the training set, under some notions of regularisation, and investigating the performance of this model on the expected risk or Test Risk, i.e., on the expected value of the loss function taken with respect to a new data point.

In the distributed setting, data is stored and processed in different locations by different agents. Each agent is represented by a node in a graph, and communication is allowed between neighbouring agents in this graph. In the decentralised setting typical of peer-to-peer networks, there is no central authority that can aggregate information from all the nodes and coordinate the distribution of computations. In sensor networks, for instance, data is collected on different sensors and each sensor communicates with nearby sensors by sharing model parameters. In the setting where the distributed data is assumed to be generated from the same unknown distribution, the goal is to design iterative algorithms so that agents can leverage local exchange of information to learn models that have better prediction capabilities as compared to the models they would obtain by only using the data they own.


In recent years, primarily due to the explosion in the size of modern datasets, the decentralised nature in which modern data is collected, and the rise of distributed computing platforms, the setting of distributed machine learning has received increased attention. From an optimisation point of view, problems in decentralised multi-agent learning are typically treated as a particular instance of consensus optimisation, and a variety of techniques have been developed to address this general framework, starting from the early work of \cite{tsitsiklis1984problems,tsitsiklis1986distributed} to more recent work that relates to the setting that we consider, which includes \cite{johansson2007simple,nedic2009distributed,Nedic2009,johansson2009randomized,ram2010distributed,lobel2011distributed,matei2011performance,Boyd:2011:DOS:2185815.2185816,DAW12,shi2015extra,mokhtari2016dsa}. From a statistical point of view, however, as emphasised in \cite{shamir2014distributed}, distributed learning problems have more structure than general consensus problems, due to the possible statistical similarities in the data owned by different agents, for instance. Aside from the client-server (star network) setting where a central aggregator can coordinate the exchange of information with every other node so that divide and conquer protocols are admissible \citep{Lin:255609}, the literature on statistical guarantees for distributed methods seems to have focused exclusively on the investigation of models with explicit regularisation, i.e., when constraints and/or penalty terms are added to the minimisation of the empirical loss function \citep{agarwal2011distributed,zhang2012communication,shamir2014communication,zhang2015disco,BSS17}.
The presence of explicit regularisation typically increases the complexity of both the algorithms and the resulting theoretical analysis, particularly for the distributed setting. For example, constraints can require the need for projection steps which are potentially costly for low-powered sensors, and deriving error bounds that depend on the graph topology for distributed algorithms in the presence of constraints is known to be challenging \citep{DAW12}.
We are not aware of any result that investigates the performance of distributed and decentralised algorithms (i.e., not divide and conquer methods) on the Test Risk in the absence of explicit regularisation. This is in sharp contrast with the centralised setting, where recent progress has been made giving optimal statistical learning guarantees for algorithms based on unregularised empirical risk minimisation via notions of implicit regularisation, i.e., proper tuning of algorithmic parameters \citep{ying2008online,tarres2014online,dieuleveut2016nonparametric,lin2016generalization,lin2017optimal}. 

%

\subsection{Contributions}
This paper investigates the learning capabilities of a simple distributed first-order method for multi-agent learning using notions of implicit regularisation that depend on the topology of the underlying communication graph. We consider the unconstrained and unpenalised empirical risk minimisation problem in the setting where $n$ agents have access to $m$ independent data points coming from the same unknown distribution, and where agents can only exchange information with their neighbours. We consider a distributed version of stochastic subgradient descent (Distributed SGD), which is a stochastic variant of one of the most widely-studied first-order method in multi-agent optimisation \citep{nedic2009distributed}. In the implementation that we look at, at every iteration each agent first performs a standard SGD step, where only one data point is uniformly sampled with replacement among those individually-owned to compute the local subgradient, and then performs a classical consensus step, where a local exchange of information is implemented via an average of the updated iterates across neighbouring agents. We treat Distributed SGD as a \emph{statistical} device, and look at its performance on unseen data by bounding the Test Error, i.e., the expected value of the excess risk defined as the difference between the Test Risk evaluated at the output of the algorithm and the minimal Test Risk. Under different assumptions on the convex loss function (we consider the  standard assumptions of Lipschitz and smoothness) we establish upper bounds for the Test Error of Distributed SGD that exhibit explicit dependence on both the algorithmic tuning parameters (the learning rate and the time horizon) and the graph topology (the spectral gap of the communication matrix). Minimising these upper bounds yields implicit regularisation strategies, allowing to recover the centralised statistical rates by proper tuning of the learning rates and of the time horizon (a.k.a.\ early stopping) as a function of the network topology. 
In the case of convex, Lipschitz, and smooth losses, we recover, up to logarithmic terms, the optimal rate of $O(1/\sqrt{nm})$ for single-pass constrained centralised SGD \citep{Lan2012,xiao2010dual}. 
In the case of convex and Lipschitz losses, we recover, up to logarithmic terms, the best-known rate of $O(1/(nm)^{1/3})$ for centralised SGD with implicit regularisation \cite{lin2016generalization,LRX16}.\footnote{
\cite{LRX16} considers implicit regularisation for gradient descent, although they remark that the analysis can be modified to account for stochastic gradients. 
}
We present numerical experiments to show that the qualitative nature of the upper bounds we derive can be representative of real behaviours.

To establish learning rates for Distributed SGD, we follow the general framework pioneered in the centralised setting by  \cite{bousquet2008tradeoffs} and, in particular, by \cite{Hardt:2016:TFG:3045390.3045520}. We consider, in the distributed setting, a decomposition of the Test Error which involves the Generalisation Error (i.e., the expected value of the difference between the loss incurred on the training data versus the loss incurred on a new data point) and the Optimisation Error (i.e., the expected value of the error on the training set).
To bound the Generalisation Error, we use algorithmic stability or sensitivity as initially put forward by \cite{Bousquet:2002} and later applied for centralised stochastic subgradient descent in \cite{Hardt:2016:TFG:3045390.3045520}. The notion of stability that we use measures how much the output of an algorithm differs when a single observation is resampled. In our case, as the observations are spread throughout the communication graph, we need to consider stability not only with respect to time (i.e., the iteration time of the algorithm), but also with respect to space (i.e., the communication graph). This technology allows us to establish generalisation bounds for Distributed SGD that do not depend on the topology of the communication graph, and we recover the same type of results that hold in the centralised setting. 
This is in contrast to optimisation bounds for distributed subgradient methods, which typically depend on the graph topology, as initially seen in the work of \cite{johansson2009randomized,johansson2007simple,DAW12}. To bound the Optimisation Error, we follow the approach pioneered in \cite{nedic2009distributed} and compare the behaviour of Distributed SGD with its network average, and we take inspiration from the analysis of the network term in the work of \cite{DAW12} (in the case of dual methods for constrained problems with Lipschitz losses) to derive upper bounds that depend on the graph topology via the inverse of the spectral gap of the communication matrix. In our setting, as we investigate implicit regularisation strategies, we deal with unconstrained problems and the evolution of the network-averaged process admits a simple form that facilitates the analysis. This approach avoids the difficulties with the nonlinearity of projection that have been previously challenging in distributed learning models, and that motivated the investigation of dual methods such as in \cite{DAW12}.
The bounds that we establish for the Optimisation Error of Distributed SGD seem novel and are of independent interest.

Finally, our results show that one can also think of the graph itself as a regularisation parameter. To give an example, agents can achieve the same statistical guarantees by trading off communication against iterations: they can choose to communicate by using a low-energy sparse communication protocol per iteration (for instance, communicating using a grid-like protocol even if the underlying topology is that of a complete graph and all agents are connected with each others), but would need to communicate for a longer time horizon in order to be guaranteed to reach the same level of statistical accuracy.

The main contributions of this work are here summarised.
\begin{enumerate}
\item \textbf{Graph-dependent implicit regularisation.} We propose graph-dependent implicit regularisation strategies for problems in distributed machine learning, specifically, step size tuning and early stopping as a function of the spectral gap of the communication matrix. Our results also show that the graph itself can be interpreted as a regularisation parameter.
\item \textbf{Optimal statistical rates using a simple algorithm.} Using implicit regularisation, we show how a simple, primal, unconstrained, first-order method (Distributed SGD) recovers, up to logarithmic terms, centralised statistical rates, in particular matching the optimal rates in the case of smooth loss functions for constrained single-pass Centralised SGD.
\item To establish statistical rates and control the Test Error of Distributed SGD, we use a distributed version of the error decomposition proposed in \cite{Hardt:2016:TFG:3045390.3045520}. We establish error bounds on the Generalisation Error and Optimisation Error, respectively.
\begin{enumerate}
\item \textbf{Distributed generalisation bounds.} We establish graph-independent Generalisation Error bounds for Distributed SGD that match those within \cite{Hardt:2016:TFG:3045390.3045520} for the centralised case. In the case of convex losses that are Lipschitz and smooth, we prove upper bounds that grow linearly with the number of iterations and step size. 
\item \textbf{Distributed optimisation bounds.} We establish graph-dependent Optimisation Error bounds for Distributed SGD. In the case of convex and Lipschitz loss functions, our analysis is inspired by \cite{nedic2009distributed,DAW12}. When smoothness is considered, our analysis is inspired by \cite{bubeck2015convex,dekel2012optimal}.
\end{enumerate}
\end{enumerate}

The remainder of the work is laid out as follows. Section \ref{sec:StabDistAlgo} introduces the framework of multi-agent learning. Section \ref{sec:DistSGD} introduces the Distributed SGD algorithm.  
Section \ref{sec:Results} presents the main results of this work, Test Error bounds for Distributed SGD with convex, Lipschitz, and either smooth or non-smooth losses. Section \ref{sec:ErrorBounds} presents the specific Generalisation and Optimisation Error bounds, as well as the notion of stability that we use. Section \ref{sec:Simulations} gives a simulation study for the case of smooth losses. Section \ref{sec:Conclusion} contains the conclusion. Appendix \ref{sec:AppGenProofs} provides proofs for all Generalisation and Test Error bounds.  Appendix \ref{sec:distOptimBounds} gives proofs for Optimisation Error bounds under a general first-order stochastic oracle model.

\section{Multi-Agent Learning}
\label{sec:StabDistAlgo}
In this section we introduce the framework of distributed and decentralised machine learning that we consider.
We address the case in which agents or nodes in a network are given their own independent datasets and they want to cooperate, by iteratively exchanging information with their neighbours, to develop a good learning model for new unseen data.

Let $(V,E)$ be a simple undirected graph with $n$ nodes, $V =\{1,\dots,n\} \equiv [n]$ being the vertex set and $E \subseteq V \times V$ being the edge set. Let $\mathcal{Z}$ be the space of observations, and to each $v \in V$ let $\mathcal{D}_v:= \{Z_{v,1},\dots,Z_{v,m}\} \in \mathcal{Z}^{m}$ 
 denote the training set associated to node $v$, which consists of $m$ i.i.d.\ data points sampled from a certain unknown distribution supported on $\mathcal{Z}$. Let $\mathcal{D}:= \cup_{v \in V} \mathcal{D}_v$ denote the collection of all data points, that is, the entire/global training dataset. Let $d>0$ be a given positive integer, and define $\mathcal{X}=\R^d$. Each agent wants to find a model $x^\star\in \mathcal{X}$ that minimises of the Test Risk $r$, which is defined as
 $$
 	r(x) := \E\,\ell(x,Z),
 $$
Here, the function $\ell:\mathcal{X} \times \mathcal{Z} \rightarrow \mathbb{R}$ is a given loss function, and $\ell(x,Z)$ represents the loss of the model $x$ on the random sample $Z$, which represents a new (unseen, independent) data point from the same distribution. We assume that the minimum of $r$ can be achieved. 
As the distribution of the data is unknown, the expected risk $r$ can not be computed, and a popular approach in machine learning is to consider the empirical risk as a proxy. In the distributed setting, the global empirical risk $R$ is defined as
$$
	R(x):=  \frac{1}{nm} \sum_{v \in V} \sum_{i=1}^m \ell(x,Z_{v,i}) 
	= \frac{1}{n} \sum_{v \in V} R_v(x).
$$
Here, we have further defined the local empirical risk $R_v(x):= \frac{1}{m} \sum_{i=1}^m \ell(x,Z_{v,i})$, for any $v \in V$. Let us denote by $X^\star\in \argmin_{x \in \mathcal{X}}R(x)$ a minimiser of the global empirical risk. In the decentralised setting that we consider, each agent $v\in V$ iteratively exchanges information with their neighbours for a certain amount of time steps $t$ to construct a model $X_v^t \in \mathcal{X}$ that can be a good proxy for the minimiser of the expected risk, i.e., for $x^\star \in \argmin_{x\in \mathcal{X}} r(x)$. A way to assess the quality of a model $X_v^t$ is to consider the Test Error, which we define as the expected value of the excess risk $r(X_v^t) - r(x^\star)$, namely,
$$
	\E\, r(X_v^t) - r(x^\star).
$$

In the next section we introduce the specific distributed algorithm that we consider to generate the models' estimates $X_v^t$'s, and we then present the main results on the bounds for the Test Error. The general paradigm that we adopt to bound the Test Error is given by a generalisation to the distributed setting of the error decomposition given in \cite{Hardt:2016:TFG:3045390.3045520} for the centralised setting. This decomposition allows to bound the Test Error of a model into the sum of two errors: the \textit{Generalisation Error}, which controls the difference between the performance of the model on a new data point and the performance of the model on the training data in $\mathcal{D}$; and the \textit{Optimisation Error}, which controls how well the model optimises the empirical risk.

\begin{proposition}[\cite{Hardt:2016:TFG:3045390.3045520}]
\label{prop:disterrorsplit}
For each $v \in V$, $t\ge 1$ we have
\begin{align*}
	\underbrace{ \E\,r(X_v^t) - r(x^\star)}_{
	\text{Test Error}
	} &\le 
	\underbrace{\E [r(X_v^t) - R(X_v^t)]}_{
	\text{Generalisation Error}
	}
	+
	\underbrace{\E[R(X_v^t) - R(X^\star) ]}_{
	\text{Optimisation Error}
	}.
\end{align*}
\end{proposition}
\begin{proof}
For completeness, the proof from \cite{Hardt:2016:TFG:3045390.3045520} is given in Appendix \ref{proof:prop:disterrorsplit}. 
\end{proof}

By using the error decomposition in Proposition \ref{prop:disterrorsplit}, we are able to consider the unregularised empirical risk minimisation problem introduced above and develop implicit regularisation strategies for a simple iterative algorithm, which we introduce next.

\begin{remark}[Statistical optimisation]
From the statistical point of view, the distributed setting where each agent is given a subset of the data has received a lot of attention in the literature (see introduction), though most of the literature on statistical optimisation has focused on the client-server (also known as master-slave) architecture typical of data centers, where a central aggregator in the network (the server) can communicate with every other nodes (the clients) and can thus coordinate the processing and exchange of information. This amounts to a star network topology that can be used to model shared-memory protocols. This type of architecture makes divide-and-conquer strategies possible, and most of the literature on statistical optimisation has focused on investigating statistical rates on the Test Error for one-shot-averaging techniques. In this work, we focus on the decentralised setting where all nodes iteratively perform the same type of computations and communications with respect to the underlying graph structure, without the presence of any special node. We are not aware of any prior work that directly investigates the statistical performance of decentralised methods on the Test Error. Most of the literature on decentralised methods seem to have focused on bounding the Optimisation Error on the training data, as we explain in Remark \ref{rem:consensus}.
\end{remark}

\begin{remark}[Consensus optimisation]
\label{rem:consensus}
From the optimisation point of view, the literature on multi-agent learning has largely focused on the investigation of the Optimisation Error via consensus methods in the presence of explicit regularisation, typically in the form of a convex constraint set $\mathcal{R}$ (see literature review in the introduction). Statistically, this approach is justified, for instance, by the distributed version of the classical error decomposition given in \cite{bousquet2008tradeoffs}:
\begin{align*}
	\underbrace{ \E\, r(X_v^t) - r(x^\star)}_{
	\text{Test Error}
	} &\le 
	2\underbrace{\E \sup_{x\in\mathcal{R}}|r(x) - R(x)|}_{
	\text{Uniform Generalisation Error}
	}
	+
	\underbrace{\E[R(X_v^t) - R(X^\star_\mathcal{R}) ]}_{
	\text{Regularised Optimisation Error}
	}
	+
	\underbrace{r(x^\star_\mathcal{R}) - r(x^{\star})}_{
	\text{Approximation Error}
	},
\end{align*}
with $x^\star_\mathcal{R}\in\argmin_{x\in\mathcal{R}}r(x)$ and $X^\star_\mathcal{R}\in\argmin_{x\in\mathcal{R}}R(x)$. In this setting, consensus optimisation deals with algorithms that minimise the quantity $R(X_v^t) - R(X^\star_\mathcal{R})$, where $R(x)=\frac{1}{n} \sum_{v \in V} R_v(x)$. Bounds on the Regularised Optimisation Error can then be combined with bounds on the Uniform Generalisation Error using notions of complexity for the constraint set $\mathcal{R}$ (e.g., VC dimension, Rademacher complexity, etc.).
As highlighted in \cite{shamir2014distributed}, and as we mentioned in the introduction, however, distributed learning problems have more structure than general consensus problems, as the local functions $R_v$ are random and have a specific design. In this work, we analyse a stochastic algorithm that is tailor-made for distributed learning problems (not for general consensus problems), and use the error decomposition in Proposition \ref{prop:disterrorsplit} to develop implicit regularisation strategies for the unregularised empirical risk minimisation problem.
\end{remark}


\section{Distributed Stochastic Subgradient Descent}
\label{sec:DistSGD}
The algorithm that we consider to generate the model estimates $X_v^t$'s assumes that each node $v \in V$ can query subgradients $\partial \ell$ of the loss function $\ell$ with respect to the first parameter, evaluated at points in the local dataset $\mathcal{D}_v$. We consider the stochastic setting where at each time step agent $v$ does not evaluate the full subgradient of the local empirical risk $R_v$, but instead only a subgradient $\partial\ell$ at a single randomly chosen sample in the locally-owned dataset $\mathcal{D}_v$. This is well tailored to situations where $m$ is large, as this reduces the per-iteration complexity to a constant factor.

The algorithm is defined as follows. 
Let $\partial \ell(x,Z_{v,k})$ represent an element of the subgradient of $\ell(\,\cdot\,,Z_{v,k})$ at $x$, with $k \in \{1,\ldots,m\} \equiv [m]$. Let $P \in \mathbb{R}^{n \times n}$ be a doubly stochastic matrix supported on the graph $(V,E)$, that is, $P_{ij} \not= 0$ only if $\{i,j\} \in E$. Distributed stochastic subgradient descent (Distributed SGD) generates a collections of vectors $\{X^s_v\}_{v \in V, s \geq 1}$ in $\mathcal{X}$ as follows. Given initial vectors $\{X^1_v\}_{v\in V}$, possibly random, for $s \geq 1$,
\begin{align}
	X^{s+1}_v = \sum_{w\in V} P_{vw} (X^{s}_w - \eta \partial \ell(X^{s}_w,Z_{w,K^{s+1}_w})),
	\label{alg:ditributedSGD}
\end{align}
where for each $v\in V$, $\{K^2_v,K^3_v,\ldots\}$ is a collection of i.i.d.\ random variables uniform in $[m]$, and $\eta >0$ is the step size. The above algorithm can be described as performing two steps: a stochastic gradient update $Y^{s+1}_w = X^{s}_w - \eta \partial \ell(X^{s}_w,Z_{w,K^{s+1}_w})$, and a consensus step $\sum_{w\in V} P_{vw} Y^{s+1}_w$. This framework for decentralised optimisation (albeit for a slightly different protocol, see remark \ref{rem:AltProtocol}) has been largely explored with the early works of \cite{nedic2009distributed,ram2009distributed,lobel2011distributed,DAW12}.
The fact that we consider implicit regularisation strategies allows us to focus on the unconstrained risk minimisation problem. In turn, this allows us to consider an algorithm that is much simpler to analyse than the ones previously considered in the literature, avoiding projections or dual approaches (see introduction for the relevant literature review). We also highlight the randomised sampling mechanism in algorithm \eqref{alg:ditributedSGD}, which is tailor-made for the machine learning problem at hand and not for generic consensus problems.

\begin{remark}
\label{rem:AltProtocol}
In the stochastic setting, the protocol put forward by \cite{nedic2009distributed} updates the iterates as $X^{s+1}_v = \sum_{w \in V}P_{vw} X^s_w - \eta \partial \ell(X^{s}_v,Z_{v,K^{s+1}_v})$, which is slightly different from the protocol that we consider where also the gradients are averaged across neighbours. The two main motivations for the original protocol are that it is fully decentralised, in that nodes are only required to communicate locally, and that it reduces to a consensus protocol to solve network averaging problems when $\ell = 0$. The protocol \eqref{alg:ditributedSGD} that we consider preserves these properties and it makes the error analyses simpler. The difference between these two protocols in a general setting has been investigated in the literature, see \cite{sayed2014adaptive} for instance.
\end{remark}
In the next section we present results on the performance of Distributed SGD under various assumptions on the loss function $\ell$.

\section{Results}
\label{sec:Results}
This section presents the main results of this work: Test Error bounds for Distributed SGD with smooth and non-smooth losses, Section \ref{subsec:smoothlosses} and Section \ref{subsec:nonsmoothlosses}, respectively.

Henceforth, let $\|\,\cdot\,\|$ be the $\ell_2$ norm. A function $f:\R^d\rightarrow\R$ is said to be $L$-Lipschitz, with $L> 0$, if $|f(x) - f(y)| \leq L \| x - y \|$ for all $x,y\in\R^d$, and $\beta$-smooth, with $\beta> 0$, if $\| \nabla f(x) - \nabla f(y)\| \leq \beta \| x - y\|$ for all $x,y\in\R^d$.
Let $\sigma_2(P)$ be the second largest eigenvalue in absolute value for the matrix $P$. Unless stated otherwise, we use the big-O notation $O(\,\cdot\,)$ to denote order of magnitudes up to constants in $n$ and $m$, and the notation $\widetilde{O}(\,\cdot\,)$ to denote order of magnitudes up to both constants and logarithmic terms in $n$ and $m$.
Equality modulo constants and  logarithmic terms is denoted by $\simeq$.

\subsection{Smooth Losses}
\label{subsec:smoothlosses}
We analyse the statistical rates for smooth losses. First, we present the Test Error bound in its full form. Then, we present a corollary that summarises the order of magnitudes of the bounds obtained under different choices of implicit regularisation, tuning the step size and the stopping time as a function of the graph topology. Full details are given in Appendix \ref{sec:AppGenProofs}.

For smooth losses, we present a bound that depends on both the variance of the gradient estimates and the statistical deviations between the local empirical losses $\{R_v\}_{v \in V}$. Let $\sigma,\kappa > 0$ be such that the following holds for any $v \in V$ and $s \geq 1$,
\begin{align}
\label{assumption:gradientvariance}
\E \big[ \|\nabla \ell(X^{s}_v,Z_{v,K^{s+1}_v}) - \nabla R_v(X^{s}_v)\|^2  \big] 
& \leq \sigma^2,\\
\label{assumption:gradientvariation}
\E \big[ \|\nabla \ell(X^s_v,Z_{v,K^{s+1}_v}) - \frac{1}{n} \sum_{w \in V} \nabla R_{w}(X^{s}_w)\|^2 \big] 
& \leq \kappa^2.
\end{align}
The quantity $\sigma^2$ in \eqref{assumption:gradientvariance} yields a uniform control on the variance of the stochastic gradients, while the quantity $\kappa^2$ in \eqref{assumption:gradientvariation} yields a uniform control on both the variance of the gradients and the deviation between local objectives. 
Note that if $\ell(\,\cdot\,,z)$ is $L$-Lipschitz for any $z\in\mathcal{Z}$, then both $\sigma^2$ and $\kappa^2$ are bounded by $4 L^2$ by the triangle inequality. A detailed discussion of these assumptions is given in Appendix \ref{sec:distOptimBounds} in the more general context of stochastic optimisation.
\begin{theorem}[Test Error bounds for convex, Lipschitz, and smooth losses]
\label{thm:convexrisknmin}
Assume that for any $z\in\mathcal{Z}$ the function $\ell(\,\cdot\,,z)$ is convex, $L$-Lipschitz, $\beta$-smooth and satisfies 
\eqref{assumption:gradientvariance} and \eqref{assumption:gradientvariation}. Let $X^1_v = 0$ for all $v \in V$, $\| X^\star \| \leq G$. 
Then, Distributed SGD with $\eta = 1/(\beta + 1/\rho)$, $\rho>0$, and $ \eta  \beta \leq 2 $, yields, for any $v\in V$ and $t \geq 1$,
\begin{align*}
	& \E\,r\Big( \frac{1}{t} \sum_{s=1}^{t} X^{s+1}_v \Big)  -  r(x^\star) \le 
	\underbrace{\frac{L^2}{nm ( \beta + 1/\rho) } (t+1)}_{\text{Generalisation Error bound}}\\
	 & \quad \quad + 
	 \underbrace{
	\frac{\rho}{2} \sigma^2 + \frac{(\beta + 1/\rho) G^2}{2t} + 
	\frac{3 \kappa}{\beta + 1/\rho} \frac{\log((t+1)\sqrt{n})}{1-\sigma_2(P)} 
	\Big( 
	L + 
	\frac{3}{2} \frac{ \beta(3+ \beta \rho)\kappa }{\beta + 1/\rho} \frac{\log((t+1)\sqrt{n})}{1-\sigma_2(P)} \Big) }_{\text{Optimisation Error bound}}.
\end{align*}
\end{theorem}
\begin{proof}
See Appendix \ref{app:proof:TestError}.
\end{proof}
We highlight that the Generalisation Error bound is independent of the graph topology, while the Optimisation Error bound naturally depends upon inverse of the spectral gap of the communication matrix: $1/(1-\sigma_2(P))$.  The following corollary gives the order of magnitudes for the Test Error bounds obtained with three different choices of step size and corresponding early stopping. The different choices for the parameter $\rho>0$ correspond to the following (modulo the simplifications used to perform the minimisations, as explained in detail in Section \ref{App:Calculations:Smooth}):
\begin{itemize}
\item $\rho^\star$ is the choice for centralised SGD, see for instance \cite{dekel2012optimal,bubeck2015convex};
\item $\rho^\star_{\mathrm{Opt}}$ is the choice that minimises the Optimisation Error bound in Theorem \ref{thm:convexrisknmin};
\item $\rho^\star_{\mathrm{Test}}$ is the choice that minimises the Test Error bound in Theorem \ref{thm:convexrisknmin}.
\end{itemize}
\begin{corollary}[Implicit regularisation for convex, Lipschitz, and smooth losses]
\label{Cor:smoothrates}
In the setting of Theorem \ref{thm:convexrisknmin}, the following holds for different choices of $\rho$, function of the time horizon $t$:
\begin{table}[!h]
\centering
\begin{tabular}{l|l|l|l}
\ $\rho$ \!\!\! & Size  &  Test Error at $\rho,t$ & Test Error at $\rho,t^\star(\rho)$\\
\hline
 \rule{0pt}{16pt} $\rho^{\star}$ & $ O \Big(\frac{1}{\sqrt{t}} \Big) $ & 
$\widetilde O\Big( \frac{1}{(1-\sigma_2(P))\sqrt{t}} + \frac{\sqrt{t}}{nm} \Big)$ 
& $\widetilde O\Big(\frac{1}{\sqrt{nm (1-\sigma_2(P))}}\Big)$ \\
\hline
 \rule{0pt}{16pt} $\rho^{\star}_{\mathrm{Opt}}$ 
 & $ \widetilde{O} \Big(\sqrt{ \frac{1-\sigma_2(P) }{t}} \Big) $ 
 &  $\widetilde{O} \Big( \frac{1}{\sqrt{t(1-\sigma_2(P)}} 
 + \frac{\sqrt{t(1-\sigma_2(P))}}{n m} \Big)$ 
 & $\widetilde{O} \Big(\frac{1}{\sqrt{nm}}\Big)$ \\
\hline
 \rule{0pt}{22pt} $\rho^{\star}_{\mathrm{Test}}$ & 
$ \widetilde{O} 
\Big( \frac{1}{\sqrt{t}} \frac{1}{\sqrt{ \frac{9}{1-\sigma_2(P)} + \frac{t}{ nm}}} \Big) $ 
& $\widetilde{O}\Big( \frac{1}{\sqrt{t(1-\sigma_2(P) }} + \frac{1}{\sqrt{nm}} \Big)$  
&  $\widetilde{O} \Big(\frac{1}{\sqrt{nm}} \Big)$ \\
\end{tabular}
\caption{$t^\star(\rho^\star) \simeq t^\star(\rho^{\star}_{\mathrm{Opt}}) \simeq
t^\star(\rho^{\star}_{\mathrm{Test}}) \simeq   nm/(1-\sigma_2(P))$.}
\label{table:StepSizeChoices}
\end{table}
\FloatBarrier
\end{corollary}

\begin{proof}
See Appendix \ref{App:Calculations:Smooth}.
\end{proof}

We note that the Test Error bound given by the choice $\rho^{\star}_{\mathrm{Test}}$ is the only one that is guaranteed to converge as the number of iterations $t$ goes to infinity. With this choice, $t^\star(\rho^\star_{\mathrm{Test}}) \simeq nm/(1-\sigma_2(P))$ iterations are guaranteed to reach the rate $\widetilde{O}(1/\sqrt{nm})$. Minimising (approximately) with respect to time the Test Error bounds that are obtained with the choices $\rho^{\star}$ and $\rho^{\star}_{\mathrm{Opt}}$ gives early stopping rules with the same order of iterations, i.e., $t^\star(\rho^\star) \simeq t^\star(\rho^\star_{\mathrm{Opt}}) \simeq nm/(1-\sigma_2(P))$. The choices $\rho^{\star}_{\mathrm{Test}}$ and $\rho^{\star}_{\mathrm{Opt}}$ with early stopping yield, up to logarithmic terms, the optimal rate ${O}(1/\sqrt{nm})$ for single-pass constrained centralised SGD \citep{Lan2012,xiao2010dual}. On the other hand, the choice $\rho^{\star}$ that aligns with centralised SGD, with no dependence on the graph topology, yields a suboptimal statistical guarantee with a rate $\widetilde{O}(1/\sqrt{nm (1-\sigma_2(P))})$. 

\subsection{Non-Smooth Losses}
\label{subsec:nonsmoothlosses}
We now analyse the statistical rates for non-smooth losses. Before presenting the results, we introduce and motivate the technical assumptions that we need. 
\begin{assumptions}
\label{assumptions}
$ $
\begin{enumerate}[(a)]
\item There exist constants $C\le B$ such that for any $z \in \mathcal{Z}$  the loss function $\ell(\,\cdot\,,z)$ is bounded from above at zero, i.e., $\ell(0,z) \leq B$, and is uniformly bounded from below, i.e., $C \le \ell(x,z)$ for any $x \in \mathbb{R}^d$;
\item There exists a constant $D \geq 0$ such that for any $z_1,\ldots,z_{nm}\in\mathcal{Z}$ and any $\mathcal{\widetilde X}\subseteq\mathcal{X}$ we have
\begin{align*}
\E \sup_{x \in \mathcal{\widetilde X}} \frac{1}{nm} \sum_{i=1}^{nm} \sigma_i \ell(x,z_i) \leq
D\frac{\sup_{x\in\mathcal{\widetilde X}} \| x \|}{\sqrt{nm}},
\end{align*}
where $\{\sigma_i\}_{i \in [nm]}$ is a collection of independent Rademacher random variables, namely, $\mathbf{P}(\sigma_i = 1) = \mathbf{P}(\sigma_i = -1) = 1/2$.
\end{enumerate}
\end{assumptions}
Assumption \textit{(a)} is a common boundedness assumption for controlling the norm of the iterates of gradient descent algorithms through a centring argument.
Assumption \textit{(b)} represents a control on the Rademacher complexity of the function class $\{\ell(x,\,\cdot\,): x\in\mathcal{X}\}$ with respect to the $\ell_2$ norm. These assumptions are satisfied, for instance, in the setting of supervised learning with linear predictors, bounded data, and hinge loss (with is convex, Lipschitz, and non-smooth). See Remark \ref{remark:NonsmoothAssumption} below.

First, we present the Test Error bound for non-smooth losses under Assumptions \ref{assumptions}. Then, we present a corollary that summarises the order of magnitudes of the bounds obtained under different choices of implicit regularisation, tuning the step size and the stopping time as a function of the graph topology. Full details are given in Appendix \ref{sec:AppGenProofs}.

\begin{theorem}[Test Error bounds for convex and Lipschitz losses]
\label{thm:convexrisknmin_nonsmooth} 
Assume that for any $z \in\mathcal{Z}$ the loss function $\ell(\,\cdot\,,z)$ is convex and $L$-Lipschitz. Consider Assumptions \ref{assumptions}. Let $X^1_v = 0$ for all $v \in V$, $\| X^\star \| \leq G$. Then, Distributed SGD with $\eta > 0$ yields, for any $v\in V$ and $t \geq 1$,
\begin{align*}
	\E\,r \Big( \frac{1}{t} \sum_{s=1}^t X^s_v \Big) - r(x^\star) 
	& \leq 
	\underbrace{ 
	2 D \sqrt{ \frac{ (t-1) (\eta^2L^2 + 2 \eta (B-C)) }{nm}} 
	}_{\text{Generalisation Error bound}}+ 
	\underbrace{ 
	\frac{19 }{2} \frac{\eta L^2 \log(t\sqrt{n}) }{1-\sigma_2(P) }	
	+ \frac{G^2}{2\eta t} 
	}_{\text{Optimisation Error bound}}.
\end{align*}
\end{theorem}
\begin{proof}
See Appendix \ref{app:proof:TestError}.
\end{proof}
The following corollary gives the order of magnitudes for the Test Error bound obtained with three different choices of step size and corresponding early stopping. The different choices for the step size $\eta  > 0$ correspond to the following (modulo the simplifications used to perform the minimisations, as explained in detail in Section \ref{App:Calculations:Lipschitz}):
\begin{itemize}
\item  $\eta^\star$ is the choice for centralised SGD, see for instance \cite{bubeck2015convex};
\item $\eta^\star_{\mathrm{Opt}}$ is the choice that minimises the Optimisation Error bound in Theorem \ref{thm:convexrisknmin_nonsmooth};
\item $\eta^\star_{\mathrm{Test}}$ is the choice that minimises the Test Error bound in Theorem \ref{thm:convexrisknmin_nonsmooth}.
\end{itemize}
\begin{corollary}[Implicit regularisation for convex and Lipschitz losses]
\label{Cor:lipschitzrates}
In the setting of Theorem \ref{thm:convexrisknmin_nonsmooth}, the following holds for different choices of $\eta$, function of the time horizon $t$:
\begin{table}[!h]
\begin{tabular}{l|l|l|l}
\ $\eta$ \!\!\! & Size  &  Test Error at $\eta,t$ & Test Error at $\eta,t^\star(\eta)$\\
\hline
 \rule{0pt}{16pt} $\eta^{\star}$ & $ O \Big(\frac{1}{\sqrt{t}} \Big) $ & 
$\widetilde O\Big( \frac{1}{(1-\sigma_2(P))\sqrt{t}} + \sqrt{ \frac{\sqrt{t}}{nm} } \Big)$ 
& $\widetilde O\Big(\frac{1}{ (nm (1-\sigma_2(P))^{1/3} }\Big)$  
\\
\hline
 \rule{0pt}{16pt} $\eta^{\star}_{\mathrm{Opt}}$  \!\!\!
 & $ \widetilde{O} \Big(\sqrt{ \frac{1-\sigma_2(P) }{t}} \Big) $ 
 &  $\widetilde{O} \Big( \frac{1}{\sqrt{t(1-\sigma_2(P)}} 
 + \sqrt{\frac{\sqrt{t(1-\sigma_2(P))} }{nm} }  \Big)$ 
 & $\widetilde{O} \Big(\frac{1}{ (nm)^{1/3} }\Big)$ \\
\hline
 \rule{0pt}{22pt} $\eta^{\star}_{\mathrm{Test}}$ \!\!\! & 
$ \widetilde{O} 
\Big( \frac{1}{\sqrt{t}} \frac{1}{\sqrt{ \frac{9}{1-\sigma_2(P)} + \frac{t}{ (nm)^{2/3} }}} \Big) $   
& $\widetilde{O}\Big( \frac{1}{\sqrt{t(1-\sigma_2(P) }} + \frac{1}{ (nm)^{1/3} } \Big)$  
&  $\widetilde{O} \Big(\frac{1}{ (nm)^{1/3} } \Big)$ \\
\end{tabular}
\caption{$t^\star(\eta^\star) \simeq (nm)^{2/3}/((1-\sigma_2(P))^{4/3}$ and 
$ t^\star(\eta^{\star}_{\mathrm{Opt}}) \simeq 
t^\star(\eta^{\star}_{\mathrm{Test}}) \simeq  (nm)^{2/3} /(1-\sigma_2(P))$.}
\label{table:StepSizeChoices_nonsmooth}
\end{table}
\FloatBarrier
\end{corollary}
\begin{proof}
See Appendix \ref{App:Calculations:Lipschitz}.
\end{proof}
Corollary \ref{Cor:lipschitzrates} shows asymptotic behaviours for the Test Error bounds (as a function of time $t$ upon different choices of the step size) that are analogous to the ones established in Corollary \ref{Cor:smoothrates} in the case of smooth losses. In particular, as in Corollary \ref{Cor:smoothrates}, also in this case the step sizes accounting for the graph topology, i.e., $\eta^\star_{\mathrm{Test}}$ and $\eta^\star_{\mathrm{Opt}}$, give improved statistical rates over the step size independent of the graph topology $\eta^\star$. 

The statistical rate obtained by both $\eta^\star_{\mathrm{Test}}$ and $\eta^\star_{\mathrm{Opt}}$, upon performing early stopping, matches, up to logarithmic terms, the best-known rate of $O(1/(nm)^{1/3})$ obtained by centralised SGD with implicit regularisation
\citep{lin2016generalization}. 
Differing from the smooth case, additional iterations with respect to the graph topology are required for the step size independent of the graph topology  $\eta^\star$ to achieve its best statistical rates (as prescribed by our upper bounds), when compared to  step sizes accounting for the topology $\eta^\star_{\mathrm{Test}}$ and $\eta^\star_{\mathrm{Opt}}$.
As highlighted in
\citep{lin2016generalization}, we note that these rates are not sharp, leaving it to future work to obtain better bounds. 
\begin{remark}
\label{remark:NonsmoothAssumption}
Assumptions \ref{assumptions} is satisfied in the setting of supervised learning with bounded data, linear predictors, and hinge loss, for instance. In this setting, each observation $z\in\mathcal{Z}$ decomposes into a  $d$-dimensional feature vector and a real-valued response, i.e., $z = \{w,y\}$ with $w \in \mathcal{W}\subset\mathbb{R}^{d}$ and $y \in \mathcal{Y}\subset\mathbb{R}$ such that $\|w\| \le D_\mathcal{W} < \infty$, and $|y|  \le D_\mathcal{Y} < \infty$. The linear predictors are parametrised by $x\in\mathcal{\widetilde X}\subseteq \mathcal{X}=\mathbb{R}^d$, i.e., $w \rightarrow w^{\top} x$, and the loss function is of the form $\ell(x,z)  = \tilde{\ell}(w^{\top}x,y)$ with the function $\tilde{\ell}:\mathcal{Y}\times\mathcal{Y}\rightarrow\mathbb{R}_+$ measuring the discrepancy between the predicted response $w^{\top}x$ and the observed response $y$. For the hinge loss, $\tilde{\ell}(\tilde{y},y) = \max(0,1 - \tilde{y} y)$. Assumption \ref{assumptions} \textit{(a)} is satisfied with $B=1$ and $C=0$. By Talagrand's contraction lemma and standard results on the Rademacher complexity of linear predictors, assumption \textit{(b)} is satisfied with $D = D_\mathcal{Y} D_\mathcal{W}$,
as the hinge loss $\tilde{\ell}(\,\cdot\,,y)$ is $|y|$-Lipschitz. Also the Lipschitz constant in Theorem \ref{thm:convexrisknmin_nonsmooth} reads $L=D$, as $|\ell(x_1,z)-\ell(x_2,z)| \le D_{\mathcal{Y}} |(x_1-x_2)^\top w| \le D_{\mathcal{Y}} D_{\mathcal{W}} \| x_1-x_2 \|$ by the Cauchy-Schwarz's inequality.
\end{remark}

\section{Generalisation and Optimisation Error Bounds}
\label{sec:ErrorBounds}
In this section we present the Generalisation and Optimisation Error bounds that yield the Test Error bounds presented within Section \ref{sec:Results}. Section \ref{sec:Generalisation Error} begins with the stability analysis used to derive the Generalisation Error  bounds for smooth losses. This is followed by the Generalisation Error bound for non-smooth losses in Section \ref{sec:Generalisation Error nonsmooth}. Finally, Section \ref{sec:OptError} presents Optimisation Error bounds for both classes of losses.


\subsection{Generalisation Error Bound for Smooth Losses through Stability}
\label{sec:Generalisation Error}
To bound the Generalisation Error for smooth losses we utilise its link with stability. This has previously been investigated in \cite{rogers1978finite,kearns1999algorithmic,Bousquet:2002,mukherjee2006learning,shalev2010learnability}, with \cite{Bousquet:2002} and \cite{Hardt:2016:TFG:3045390.3045520} providing the work upon which we rely. Specifically, \cite{Hardt:2016:TFG:3045390.3045520} investigated the Generalisation Error of centralised SGD in the multi-pass setting, giving, in the case of convex, Lipschitz, and smooth losses, upper bounds that grow linearly with the number of iterations and step size. The method used is algorithmic stability (or sensitivity) as introduced in \cite{Bousquet:2002}.  This method investigates the deviation of an algorithm when a single data point in the dataset $\mathcal{D}$ is resampled. By iterating through all of the observations, accounting for the deviation in each case, the Generalisation Error is then equal to the average deviation, as we see next. In our case the observations are spread throughout a graph, so the deviations of the algorithm depends on the location of the observation that is resampled.
 
For each $w \in V$ and $k \in [m]$, let $\widetilde{Z}_{w,k}$ be a resampled (independent) observation coming from the same  data distribution. 
Let  $\widetilde X(w,k)^t_{v}$ denote the output of Distributed SGD at node $v$ after $t$ iterations when the algorithm is run on the perturbed dataset $\{\mathcal{D} \backslash Z_{w,k} \} \cup \widetilde{Z}_{w,k}$ in which the $k$-th observation for node $w$, i.e., $Z_{w,k}$, is replaced by $\widetilde{Z}_{w,k}$. The Generalisation Error is then equal to the average mean deviance of the loss function evaluated at the perturbed outputs.
\begin{proposition}
\label{prop:genfromstab}
For any $v\in V$ and $t \geq 1$,
\begin{align*}
	\E[r(X^t_v) - R(X^t_v)]
	&= 
	\frac{1}{nm} \sum_{w \in V} \sum_{k=1}^m  
	\E
	[ 
	\ell( X^t_v,\widetilde Z_{w,k}) - \ell(\widetilde X(w,k)^t_v,\widetilde Z_{w,k})
	].
\end{align*}
\end{proposition}
\begin{proof}
The proof is given in Appendix \ref{proof:prop:genfromstab}.
\end{proof}
The identity in Proposition \ref{prop:genfromstab} involves a double sum over the mean deviations of the algorithm applied to locally perturbed datasets: one sum relates to the graph location where the perturbation is supported ($w$), and the other sum relates to the index of the perturbed data point at that location ($k$).
Each \emph{individual} deviation depends on the graph topology via the location of the resampled observation $w$ relative to the node of reference $v$. This dependence is captured by the bound that we give in Proposition \ref{prop:distributedstability} in Appendix \ref{sec:GenProof}, where we show that the non-expansive property of the gradient descent update in the smooth case controls the spatial propagation of the deviation across the network via the term $\sum_{s=1}^{t-1} (P^s)_{vw}$. Proposition \ref{prop:genfromstab} involves the \emph{average} across all deviations, and once the summation over ${w \in V}$ is considered, we get a final bound that increases linearly with time but does not depend on the graph topology, as we state next.
\begin{lemma}[Generalisation Error bound for convex, Lipschitz, and smooth losses]
\label{lem:GenSmoothDist}
Assume that for any $z\in\mathcal{Z}$ the function $\ell(\,\cdot\,,z)$ is convex, $L$-Lipschitz, and $\beta$-smooth. Let $X^1_v = 0$ for all $v \in V$. Then, Distributed SGD with $\eta\beta \le 2$ yields, for any $v\in V$ and $t \geq 1$,
\begin{align*}
\E [r(X^t_v)-R(X^{t}_v)] \leq \frac{2 \eta L^2}{nm} ( t - 1).
\end{align*}
\end{lemma}
\begin{proof}
See Appendix \ref{sec:prop:proofs}.
\end{proof}
For completeness, and to fully establish in the decentralised case the results derived in \cite{Hardt:2016:TFG:3045390.3045520} in the centralised case, we include in Appendix \ref{sec:prop:proofs} also the time-uniform Generalisation Error bound for the constrained and strongly-convex case. In this case, the \emph{contraction} property of the gradient descent update controls the spatial propagation of the deviation across the network via the term $\sum_{s=1}^{t-1} \iota^s (P^s)_{vw}$, for a given $\iota < 1$. Once the summation over ${w \in V}$ in Proposition \ref{prop:genfromstab} is taken, we get a final bound that does not depend on time, nor on the graph topology.
The bounds that we give are identical to the ones in \cite{Hardt:2016:TFG:3045390.3045520} for a single agent with $nm$ observations.


\subsection{Generalisation Error Bound for Non-Smooth Losses through Rademacher complexity}
\label{sec:Generalisation Error nonsmooth}
In the case of non-smooth losses we follow the approach used in \cite{lin2016generalization} for the centralised case that involves controlling the Generalisation Error by using standard Rademacher complexity's arguments through Assumption \ref{assumptions} \textit{(b)} and bounding the norm of the iterates $\|X^t_v\|$ as a function of the parameters of the algorithm.
\begin{lemma}[Generalisation Error for convex and Lipschitz losses]
\label{lem:Lipschitz:GenError}
Assume that for any $z \in\mathcal{Z}$ the loss function $\ell(\,\cdot\,,z)$ is convex and $L$-Lipschitz. Consider Assumptions \ref{assumptions}. Let $X^1_v = 0$ for all $v \in V$. Then, Distributed SGD yields, for any $v\in V$ and $t \geq 1$,
\begin{align*}
\E[r(X^t_v) - R(X^t_v)] \leq 2D \sqrt{\frac{ (t-1)( \eta^2 L^2 + 2 \eta (B-C)) }{nm}}.
\end{align*}
\end{lemma}
\begin{proof}
See Appendix \ref{app:Lipschitz:Gen}.
\end{proof}
We now go on to give Optimisation Error bounds which, once combined the Generalisation Error bounds in Section \ref{sec:Generalisation Error} and \ref{sec:Generalisation Error nonsmooth}, give the Test Error bounds presented within Section \ref{sec:Results}.

\subsection{Optimisation Error Bounds}
\label{sec:OptError}
In this section we present Optimisation Error bounds for Distributed SGD with convex, Lipschitz, and either smooth or non-smooth losses. These results follow from theorems proved within Appendix \ref{sec:distOptimBounds} under the more general setting of the first-order stochastic oracle model. We note that constants within these bounds have not been optimised.

The bounds that we derive are proved using the techniques developed in \cite{Nedic2009} and, in particular, in \cite{DAW12}, where the evolution of the algorithm $X_v^s$ is compared against the evolution of its network average $\overline X^s:=\frac{1}{n}\sum_{v\in V} X^s_v$ to derive graph-dependent error bounds. Appendix \ref{sec:distOptimBounds} contains the full scheme of the proof, along with the error decomposition into a network term, an optimisation term, and a gradient noise term (only in the smooth case). As previously emphasised, the fact that we investigate implicit regularisation strategies allows us to deal with unconstrained problems, and in this case the evolution of the network-averaged process $\overline X^s$ admits a simple form that facilitates the analysis. This approach avoids the difficulties with the nonlinearity of projection that have been previously challenging in distributed learning models, and that motivated the investigation of dual methods such as in \cite{DAW12}. 


We start with the case of Lipschitz and smooth losses. The proof for this case is inspired from the proof for centralised SGD applied to smooth objectives, specifically, Theorem 6.3 in \cite{bubeck2015convex}, itself extracted from \cite{dekel2012optimal}.
The bound that we present depends upon both the quantity $\sigma$ and the quantity $\kappa$ defined, respectively, in \eqref{assumption:gradientvariance} and \eqref{assumption:gradientvariation}.

\begin{lemma}[Optimisation Error bound for convex, Lipschitz, and smooth losses]
\label{lem:DistSGD:ConvexSmooth:OPT}  
Assume that\\for any $z\in\mathcal{Z}$ the function $\ell(\,\cdot\,,z)$ is convex, $L$-Lipschitz, $\beta$-smooth 
and satisfies \eqref{assumption:gradientvariance} and \eqref{assumption:gradientvariation}. Let $X^1_v = 0$ for all $v \in V$, $\| X^\star \| \leq G$. Then, Distributed SGD with $\eta = 1/(\beta + 1/\rho)$ and $\rho > 0$, yields, for any $v \in V$ and $t\ge 1$,
\begin{align*}
	& \E\Big[  R\Big( \frac{1}{t} \sum_{s=1}^{t} X^{s+1}_v \Big) 
	-  R(X^\star) \Big]\\
	& \quad \quad \quad \le 
	\frac{\rho}{2} \sigma^2 + \frac{(\beta + 1/\rho) G^2}{2t} + 
	\frac{3 \kappa}{\beta + 1/\rho} \frac{\log((t+1)\sqrt{n})}{1-\sigma_2(P)} 
	\Big( 
	L + 
	\frac{3}{2} \frac{ \beta(3+ \beta \rho)\kappa }{\beta + 1/\rho} \frac{\log((t+1)\sqrt{n})}{1-\sigma_2(P)} \Big).
\end{align*}

\end{lemma}
\begin{proof}
The result follows from Corollary \ref{cor:diststochopt_smooth} in Appendix \ref{sec:distOptimBounds} and from Section \ref{proof:theorem:ConvexOPT}.
\end{proof}

Next is the Optimisation Error bound for non-smooth losses, inspired from \cite{DAW12}.
\begin{lemma}[Optimisation Error bound for convex and Lipschitz losses]
\label{lem:DistSGD:Convex:OPT}
Assume that for any $z\in\mathcal{Z}$ the function $\ell(\cdot,z)$ is convex and $L$-Lipschitz. 
Let $X^1_v = 0$ for all $v \in V$, $\| X^\star \| \leq G$. Then, Distributed SGD yields, for any $v\in V$ and $t\ge 1$,
$$
	\E\Big[ R \Big( \frac{1}{t} \sum_{s=1}^t X^s_v \Big) - R(X^\star) \Big]
	\le 
	\frac{\eta L^2}{2} \Big( 19  
	\frac{\log (t\sqrt{n})}{1-\sigma_2(P)} \bigg)
	+ \frac{G^2}{2\eta t}.
$$
\end{lemma}
\begin{proof}
The result follows from Corollary \ref{cor:diststochopt} in Appendix \ref{sec:distOptimBounds} and from Section \ref{proof:theorem:ConvexOPT}. 
\end{proof}
When optimising either of these bounds with respect to $\rho$ or $\eta$, a rate no better than $O(1/\sqrt{t})$ can be obtained, matching the rate of stochastic gradient descent in the centralised case. From the bound in Lemma \ref{lem:DistSGD:ConvexSmooth:OPT}, however, we note that if $\sigma=\kappa=0$ then the accelerated rate of $O(1/t)$ can be obtained, matching the rate of full-gradient descent in the centralised case. For a general discussion on these lines, we refer to Appendix \ref{sec:distOptimBounds} and to Remark \ref{rem:comparisonToCentralised} in particular.

\section{Numerical Experiments}
\label{sec:Simulations}

In this section we provide a simulation study to investigate if the previously proven bounds can be representative of real behaviours. Specifically, we investigate the Test Error bounds given in Corollary \ref{Cor:smoothrates} for convex, Lipschitz, and smooth losses.
We start by introducing the notation and quantities of interest in Section \ref{sec:simulations:setup}, then we present the results of the experiments in Section \ref{sec:simulations:smooth}. 
\subsection{Setup}
\label{sec:simulations:setup}
As we want to minimise the expected risk $r(x) = \E\, \ell(x,Z) $ but a closed form expression is typically not available, we use a Monte Carlo approximation constructed from an independent out of sample dataset $\{Z_{j}\}_{j\in[\widehat N]}$, namely, $\hat r(x):=\frac{1}{\widehat N} \sum_{j=1}^{\widehat N} \ell(x,Z_j)$. Given $t$ iterations of the Distributed SGD algorithm, we denote the ergodic average of the iterates by $\widehat{X}^t_v:= \frac{1}{t} \sum_{s=1}^t X^s_v$, for $v \in V$. We investigate the Out of Sample Risk defined as $\max_{v \in V} \hat r(\widehat{X}^t_v )$, which is set to be a proxy for the Test Risk for Distributed SGD, as defined in Section \ref{sec:StabDistAlgo}. We recall that the Test Error is defined as the expectation of the Test Risk minus the minimum expected risk $r(x^\star)$, which is a constant. Therefore, modulo a constant shift, Out of Sample Risk is also a proxy for the Test Error.

Given a graph $(V,E)$ with $n=|V|$ nodes, let $A\in \mathbb{R}^{n \times n}$ be its adjacency matrix defined as $A_{vw} := 1$ if $\{v,w\} \in E$ and $A_{vw} := 0$ otherwise. For each $v \in V$, let $d_v = \sum_{w\in V} A_{vw}$ denote the degree of node $v$, $d_{\text{max}}=\max_{v\in V} d_v$ the maximum degree, and $D = \text{diag}(d_1,\dots,d_n)$ the diagonal degree matrix. 
We consider the doubly stochastic matrix $P = I - \frac{1}{d_{\text{max}} + 1}(D - A)$. This choice is standard in distributed optimisation (see \cite{NET-014}, for instance).
In this case, the spectral gap is known to be of the following orders
(see \cite{DAW12}, for instance):
$$
O\left(\frac{1}{\sqrt{1-\sigma_2(P)}}\right) = 
\begin{cases}
	n & \text{Cycle}\\
	\sqrt{n} & \text{Grid} \\
	1 & \text{Complete Graph}
\end{cases}
$$
We adopt the following parametrisation: $O(1/\sqrt{1-\sigma_2(P)}) = O(n^{\alpha})$, for $\alpha \ge 0$. 

We consider an instance of logistic regression in supervised learning, where for a given positive integer $d$, we have $Z = \{W,Y\}$ with the feature vector $W \in \mathbb{R}^{d}$ and the label $Y \in \{-1,1\}$, and the parameter of interest is $X \in \mathbb{R}^d$. The loss function in this case is given by
$$
	\ell(X,Z) = \log(1 + e^{-Y \times \langle X, W \rangle}),
$$
where $\langle X, W \rangle = X^\top W = \sum_{i=1}^d X_iW_i$.
Given the node count $n$ and $m=2$ locally-owned data points, a simulated dataset with a total of $N=mn = 2n$ observations $\{Z_i\}_{i\in [N]}$ are sampled following the experiment within \cite{DAW12}. Specifically, a true parameter $X^{\star\star}$ is sampled from a standard $d$-dimensional Gaussian $\mathcal{N}(0,I)$, the feature vectors $W_i$'s are sampled uniformly from the unit sphere $\{w \in \mathbb{R}^d: \|w\| \leq 1\}$, and the responses are set as $Y_{i} = \text{sign}( \langle W_i, X^{\star\star} \rangle)$ where $\text{sign}(a) =1 $ if $a \geq 0$ and $-1$ if $a < 0$.
The dataset is then randomly spread across the graph with each node getting $m=2$ samples. It can easily be seen that the Lipschitz parameter is $L=1$ and the smoothness parameter is $\beta=1/4$. Parameters depending upon the gradient noise were upper bounded by distribution-independent quantities and set to $\sigma^2 \rightarrow 4 L^2$ and $\kappa \rightarrow L$ (see Proposition \ref{prop:network} in Appendix \ref{sec:distOptimBounds} for the interplay between $L$ and $\kappa$ as far as bounding the network term is concerned). A solution $X^\star$ to the empirical risk minimisation rule is calculated with tolerance $10^{-15}$ using the \texttt{lbfgs} solver within the \texttt{LogisticRegression} function of the python library \texttt{scikit} \citep{scikit-learn}. We set $G = \|X^\star\|$. Dimension and Monte Carlo estimate size are set to $d=100$ and $\widehat N = 1000$, and Distributed SGD is run for 15 different time horizons $t$, between $10^2$ and  either $10^{7}$ or $10^{6.5}$ for graph sizes $n=3^2 $ or $n=10^2$, respectively.
All runs are initialised from $X_{v}^1 = 0$ for all $v \in V$. Comparisons are made for three choices of the step size, as prescribed in Corollary \ref{Cor:smoothrates}, and for three choices of the graph topology: complete graph $(\alpha =0)$, grid $(\alpha = 1/2)$, and cycle $(\alpha = 1)$. 

\subsection{Experimental Results}
\label{sec:simulations:smooth}
Referring to the \emph{upper} bounds in Corollary \ref{Cor:smoothrates}, we outline what we expect to see plotting the Test Error against the time horizon $t$, with $\log-\log$ scales, across the three different step sizes:
\begin{itemize}
\item  $\rho^\star$ - For small $t$, linear decrease with graph-dependent intercept; for large $t$, linear increase with intercept independent of the graph topology. Minimum attained is graph-dependent;
\item $\rho^{\star}_{\mathrm{Opt}}$  - For small and large $t$, respectively, linear decrease and increase with graph-dependent intercept. Minimum attained is independent of graph topology;
\item $\rho^{\star}_{\mathrm{Test}}$ - Linear decrease with graph-dependent intercept up to a threshold independent of the graph topology. 
\end{itemize}
Figure \ref{fig:simulationResults_N3} presents $\log - \log$ plots of the Out of Sample Risk against the time horizon $t$, using the step sizes stated in Corollary \ref{Cor:smoothrates}. All of the behaviours described above, as suggested by our upper bounds, are observed.
In particular, recall that our bounds suggest the sub-optimality of the sample rate achieved by the step size aligned with centralised SGD ($\rho^\star$), as opposed to the other two choices ($\rho^{\star}_{\mathrm{Opt}}$ and $\rho^{\star}_{\mathrm{Test}}$) that depend on the graph topology. Corollary \ref{Cor:smoothrates} states that the Test Error for $\rho^\star$ yields the rate $\widetilde{O}(n^\alpha/\sqrt{nm})$, as opposed to the rate $\widetilde{O}(1/\sqrt{nm})$ achieved by the other two choices. The former rate is worse (i.e., larger) than the latter for the cycle ($\alpha=1$) and the grid ($\alpha=1/2$), while it is of the same order for the complete graph ($\alpha=0$). 
Evidence of this behaviour is observed in Figure \ref{fig:simulationResults_N3} for $n=100$, where the Out of Sample Risk related to the cycle and grid is seen to achieve a lower minimum when the step sizes that account for the graph topology are used. 

\begin{figure*}[h]
    \centering
	\includegraphics[width = 6 true in]{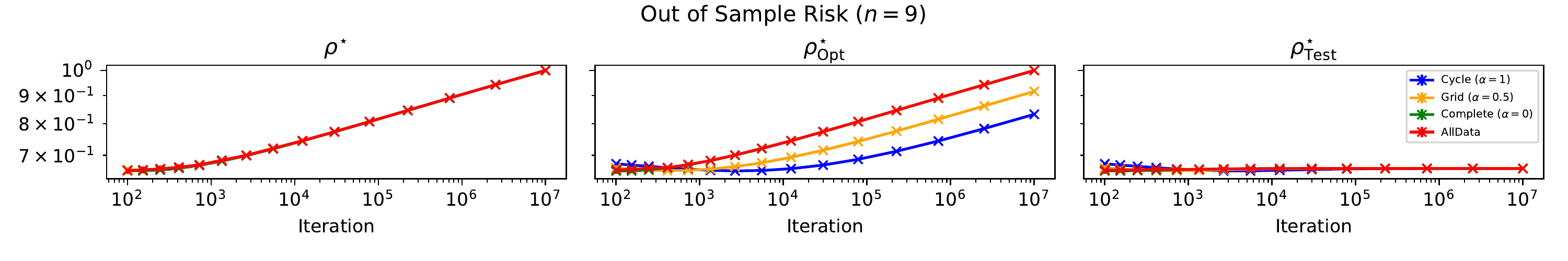}
	\includegraphics[width = 6 true in]{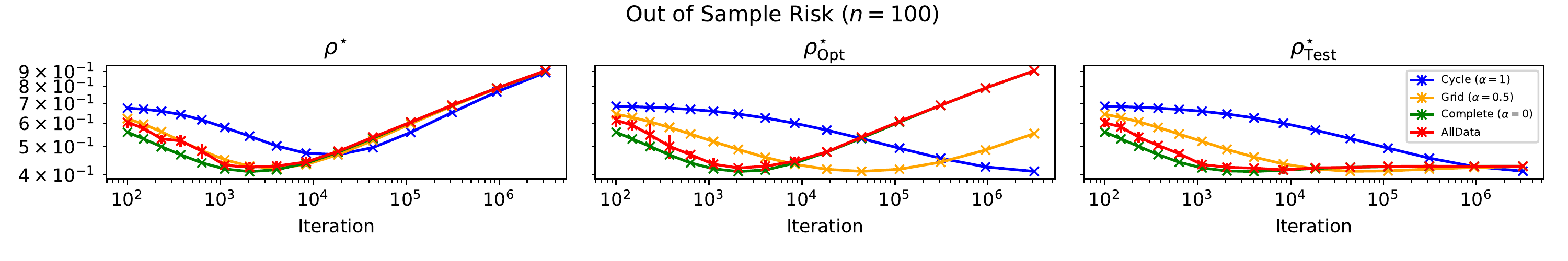}
	\caption{Out of Sample Risk against time horizon for different choices of step size: $\rho^\star$, $\rho^{\star}_{\mathrm{Opt}}$, and $\rho^{\star}_{\mathrm{Test}}$. Scales are $\log - \log$. Graph size  $n=9$ (\textit{top}), $100$ (\textit{bottom}). Simulations run for 15 values of $t$ from $10^2$ to $10^{7}$ (\textit{top}) or $10^{6.5}$ (\textit{bottom}).
	Each point  is an average over $10$ (\textit{top}) or  $4$ (\textit{bottom}) replications with error bars representing $2$ standard deviations before taking the log scale (error bars are not visible for large $t$ due to the small variance between repeated runs). \textit{AllData}: Centralised SGD run on the full dataset of $18$ (\textit{top}) or $200$ (\textit{bottom}) samples with 
$\rho^\star = \rho^\star_{\mathrm{Opt}} = O(1/\sqrt{t})$ and $\rho^\star_{\mathrm{Test}} = O( 1/(\sqrt{t} \sqrt{ 1 + t/m})$. The behaviour of Centralised SGD is seen to correspond to the behaviour of Distributed SGD on the complete graph, as expected.}
    \label{fig:simulationResults_N3}
\end{figure*}

\section{Conclusion}
\label{sec:Conclusion}
We have proposed and investigated graph-dependent implicit regularisation strategies for Distributed SGD for convex problems in multi-agent learning. Specifically, we have shown how Distributed SGD can retain centralised statistical guarantees by proper tuning of the algorithmic parameters as a function of the graph topology.
For convex, Lipschitz, and smooth losses, we showed that Distributed SGD recovers, up to logarithmic terms, the optimal rate of $O(1/\sqrt{nm})$ for single-pass constrained centralised SGD \citep{Lan2012,xiao2010dual}. 
For convex and Lipschitz losses, we showed that Distributed SGD recovers, up to logarithmic terms, the best-known rate of $O(1/(nm)^{1/3})$ for centralised SGD with implicit regularisation \citep{LRX16}.
To obtain these results we: proved Generalisation Error bounds that do not depend on the graph topology and match the bounds in the centralised setting; and derived Optimisation Error bounds that depend on the graph topology. We provided numerical simulations showing that our bounds can be representative of real behaviours. 

Our work motivates further investigation of graph-dependent implicit regularisation strategies for decentralised protocols.
Since communication is often a dominant bottleneck in distributed computations, further research is needed to investigate the improvement on the communicational and computational complexity that can be obtained by exploiting the interplay between the statistical regularities of the local objective functions and schemes involving mini-batching, acceleration, and graph sparsification. The latter relates to Gossip protocols where only a random subset of nodes communicate at each iteration \citep{5545370}. Another direction for future investigation lies in the analysis of adaptive schemes that can contemplate time-dependent step sizes and that can automatically infer the algorithmic parameters of interests, in primis the spectral gap of the communication matrix.

\appendix

\section{Proofs of Generalisation and Test Error Bounds}
\label{sec:AppGenProofs}
This appendix provides the proofs for the Generalisation and Test Error bounds presented within the main body of this paper. First, for completeness, we include the proofs of Proposition \ref{prop:disterrorsplit} and Proposition \ref{prop:genfromstab} in Section \ref{proof:prop:disterrorsplit} and Section \ref{proof:prop:genfromstab}, respectively. These results generalise to the distributed setting the Test Error decomposition and the Generalisation Error decomposition via stability used in the centralised setting, and the proofs follow the exact same arguments as in the centralised case. Second, we present the proofs of the Generalisation Error bounds for smooth and non-smooth losses in Section \ref{sec:prop:proofs} and Section \ref{app:Lipschitz:Gen}, respectively. For completeness, Section \ref{sec:prop:proofs} also includes the proof of stability for the strongly convex case with constraints, which is not covered in the main body but is here presented as it fully generalises the results in \cite{Hardt:2016:TFG:3045390.3045520} for Distributed SGD. Third, in Section \ref{app:proof:TestError} we present the proofs of Test Error bounds for smooth and non-smooth losses, referring to Theorem \ref{thm:convexrisknmin} and Theorem \ref{thm:convexrisknmin_nonsmooth} within the main body of the work. Finally, in Section \ref{App:Calculations:Smooth} and Section \ref{App:Calculations:Lipschitz} we give the calculations deriving the rates presented in Corollary \ref{Cor:smoothrates} and Corollary \ref{Cor:lipschitzrates}  for smooth and non-smooth losses, respectively.
Throughout, we use the notations $\lesssim, \simeq,\gtrsim, $ to indicate $\le,=,\ge$ modulo constants and $\log$ terms.

\subsection{Proof of Proposition \ref{prop:disterrorsplit}}
The proof is analogous to the one given in \cite{Hardt:2016:TFG:3045390.3045520} for the centralised case.\\

\label{proof:prop:disterrorsplit}
\begin{proof}[Proposition \ref{prop:disterrorsplit}]
We have
$
	r(X^t_v) - r(x^\star) =
	r(X^t_v) - R(X^t_v) +
	R(X^t_v) - R(X^\star)
	+ R(X^\star) - r(x^\star).
$
Note that $\E R(X^\star) \le r(x^\star)$, as for any $x$ we have $R(X^\star) \le R(x)$ so that
$
	\E  R(X^\star) \le \E R(x) =  r(x),
$
which holds for $x=x^\star$. Thus,
$
	\E\,r(X^t_v) - r(x^\star)
	\le \E[r(X^t_v) - R(X^t_v)] +
	\E[R(X^t_v) - R(X^\star)].
$
\end{proof}

\subsection{Proof of Proposition \ref{prop:genfromstab}}
\label{proof:prop:genfromstab}
The proof follows the ideas in \cite{Bousquet:2002} and \cite{Hardt:2016:TFG:3045390.3045520} for the centralised case.\\
\begin{proof}[Proposition \ref{prop:genfromstab}]
As the resampled observation $\widetilde Z_{w,k}$ has the same distribution than $Z$ and is independent of both $X^t_{v}$ and $
\mathcal{D}$, we have 
$
	\E\,r(X^t_v) = \frac{1}{nm} \sum_{w \in V} \sum_{k=1}^{m} 
	\E\, \ell ( X^t_v,\widetilde Z_{w,k} ).
$
As the pair $(X^t_v, Z_{w,k})$ has the same distribution as the pair $(\widetilde X(w,k)^t_v, \widetilde Z _{w,k})$, the expectation of the empirical risk can be written as
$
	\E R(X^t_v)
	= \frac{1}{nm} \sum_{w \in V}\sum_{k=1}^m 
	\E\, \ell (\widetilde X(w,k)^t_v, \widetilde Z _{w,k}).
$
Thus,
$
	\E[r(X^t_v) - R(X^t_v)]
	= 
	\frac{1}{nm} \sum_{w \in V}  \sum_{k=1}^m 
	\E
	[ 
	\ell( X^t_v,\widetilde Z_{w,k}) - \ell(\widetilde X(w,k)^t_v,\widetilde Z_{w,k})
	].
$
\end{proof}

\subsection{Proof of Generalisation Error Bounds for Smooth Losses}
\label{sec:prop:proofs}
In this section we prove the Generalisation Error bound presented in Lemma \ref{lem:GenSmoothDist} for smooth losses, and we establish a Generalisation Error bound for strongly convex functions. The proof that we present follows the spirit of the proof in \cite{Hardt:2016:TFG:3045390.3045520} for the centralised setting, using algorithmic stability. Specifically, deviations of the algorithm are studied when a single data point in the entire dataset is resampled. In the distributed setting that we consider, the training data is spread throughout the communication graph, and we need to consider stability not only with respect to time (i.e., the iteration time of the algorithm), but also with respect to space (i.e., the communication graph). As established in Proposition \ref{prop:genfromstab}, the Generalisation Error is the average of these deviations. Intermediate steps show that the individual deviations have a dependence on the graph topology, as encoded by the communication matrix $P$. However, once the average over all deviations is taken, we get results that do not depend on the graph topology.

First, in Section \ref{sec:setupstability} we describe the setup for the stability analysis. Then, in Section \ref{sec:GenProof} we present the proof for the case of convex, Lipschitz, and smooth losses. Finally, in Section \ref{sec:GenProofStrongly} we present the case of Lipschitz, smooth, and strongly-convex losses with constraints.

\subsubsection{Setup}
\label{sec:setupstability}
For any $w\in V$ and $k\in [m]$, let $\mathcal{\widetilde D}(w,k) := \{\mathcal{D} \backslash Z_{w,k} \} \cup \widetilde{Z}_{w,k}$ be the dataset in which node $w$ has the $k$-th observation resampled.  Recall that $\widetilde X(w,k)^t_v$ denotes the output at node $v$ and time step $t$ of Distributed SGD \eqref{alg:ditributedSGD} run with respect to the dataset $\mathcal{\widetilde D}(w,k)$.
From Proposition \ref{prop:genfromstab}, the link between the Generalisation Error and the $\ell_2$ deviation
$$
	\delta(w,k)^t_v := \|  \widetilde X(w,k)^t_v - X^t_v\|
$$
can be made explicit when the loss function $\ell$ is $L$-Lipschitz in the first coordinate (uniformly in the second). Specifically, each term in the double sum $\sum_{k=1}^m \sum_{w \in V}$ in Proposition \ref{prop:genfromstab} is bounded by
$$
	\ell( X^t_v,\widetilde Z_{w,k}) - \ell(\widetilde X(w,k)^t_v,\widetilde Z_{w,k}) \leq 
	L \delta(w,k)^t_v.
$$
The results that we derive directly bound the deviation $\delta(w,k)^t_v$. 
Henceforth, for a given matrix $M \in \mathbb{R}^{n \times n}$ we use the notation $M^{s}_{vw}$ to represent the quantity $(M^{s})_{vw}$, where $M^s$ is the $s$-th power of $M$, and the notation $M_v$ to represent the $v$-th row of $M$. Hence, for a given vector $x$, we write $M_v x$ to indicate $\sum_{w\in V} M_{vw}x_w$. 
For any $x,y\in\mathbb{R}^d$, we let $\langle x, y \rangle = x^\top y = \sum_{i=1}^d x_iy_i$.

Before proceeding to the main proofs we require some standard results relating to the expansive properties of gradient descent updates with smooth  and either convex or strongly convex functions. Specifically, for a sufficiently small step size, a result showing that gradient descent updates with smooth and convex function are non-expansive. Meanwhile, for additionally strongly-convex functions, a result showing that gradient descent updates are contractive. The proof
can be found in Appendix A of \cite{Hardt:2016:TFG:3045390.3045520} and it utilises the co-coercivity of gradients for smooth and convex functions \citep{nesterov2013introductory}. 
\begin{lemma}
\label{lem:gradientcontract}
Let $f$ be a $\beta$-smooth function, convex, and $\eta\beta \leq 2$ with $\eta>0$. Then, for any $x,y \in \mathbb{R}$,
\begin{align*}
\|x - y - \eta(\nabla f(x) - \nabla f(y))\| \leq \|x-y\|.
\end{align*}
Let $f$ be a $\beta$-smooth function, $\gamma$-strongly convex, and $\eta \leq 2/(\beta+\gamma)$. Then, for any $x,y \in \mathbb{R}$,
\begin{align*}
\|x - y - \eta(\nabla f(x) - \nabla f(y))\| \leq 
\Big(1 - \frac{\eta \beta \gamma}{\beta +\gamma} \Big) \|x-y\|.
\end{align*}
\end{lemma}

\subsubsection{Convex, Lipschitz, and Smooth Losses}
\label{sec:GenProof}
We start by stating Proposition \ref{prop:distributedstability} that establishes a bound on the deviation $\delta(w,k)^{t}_v$ that explicitly depends on the graph topology. This is followed by the proof of Lemma \ref{lem:GenSmoothDist}.
\begin{proposition}[Stability for convex, Lipschitz, and smooth losses]
\label{prop:distributedstability}
Assume the setting of Lemma \ref{lem:GenSmoothDist}.
Then, for any $v,w\in V,k\in [m]$ and $t \geq 1$,
$$
	\E\, \delta(w,k)^t_v
	= \E \| \widetilde X(w,k)^t_v - X^t_v \|
	\le  \frac{2 \eta L}{m} \sum_{s=1}^{t-1}  P^s_{vw}.
$$
\end{proposition}
\begin{proof}[Proposition \ref{prop:distributedstability}]
Let $\mathcal{F}_1$ be the $\sigma$-algebra generated by $\mathcal{D}$ and $\mathcal{\widetilde D}:=\{\mathcal{\widetilde D}(w,k)\}_{w\in V, k\in[m]}$. For any $t \geq 2$, let $\mathcal{F}_{t}$ be the $\sigma$-algebra generated by the datasets $\mathcal{D}$ and $\mathcal{\widetilde D}$, and by the collection of uniform random variables $\{ K^{2}_v ,\dots,K^{t}_v\}_{v\in V}$.
Plugging the algorithm updates \eqref{alg:ditributedSGD} into $\delta(w,k)^{t}_v$, applying the triangle inequality and using the fact that $\{X^{t-1}_v\}_{v\in V}, \{\widetilde X(w,k)^{t-1}_v\}_{v\in V}$, $\mathcal{D}$, and $\mathcal{\widetilde D}$ are measurable with respect to $\mathcal{F}_{t-1}$, we get
\begin{align}
& \E[\delta(w,k)^{t}_v|\mathcal{F}_{t-1}] \nonumber \\
& \leq
\sum_{l \not= w} P_{vl}
\E \Big[ 
\Big\|
\widetilde X(w,k)^{t-1}_l \!-\! X^{t-1}_l 
\!-\! \eta \Big(
\nabla \ell(\widetilde X(w,k)^{t-1}_l,Z_{l,K^{t}_l}) \!-\!  
\nabla \ell(X^{t-1}_l,Z_{l,K^{t}_l})
\Big)
\Big \| \Big| \mathcal{F}_{t-1} \Big] \label{equ:proof:smoothstab:equ1} \\
&\quad +  
\frac{P_{vw}}{m} \sum_{i \not= k}\Big\|
\widetilde X(w,k)^{t-1}_w - X^{t-1}_w 
- \eta \Big(
\nabla \ell(\widetilde X(w,k)^{t-1}_w,Z_{w,i}) -  
\nabla \ell(X^{t-1}_w,Z_{w,i})
\Big)
\Big\| \label{equ:proof:smoothstab:equ2} \\
&\quad  + \frac{P_{vw}}{m} \Big\|
\widetilde X(w,k)^{t-1}_w - X^{t-1}_w 
- \eta \Big(
\nabla \ell(\widetilde X(w,k)^{t-1}_w,\widetilde Z_{w,k}) -  
\nabla \ell(X^{t-1}_w,Z_{w,k})
\Big)
\Big\|.  \label{equ:proof:smthstability:equ3}
\end{align}
This decomposition isolates \eqref{equ:proof:smthstability:equ3}, the term relative to node $w$ and data index $k$ that is the only one to involve the difference of two gradients evaluated at different data points ($\widetilde Z_{w,k}$ and $Z_{w,k}$).
To bound this term, we use the Lipschitz property, $\|\nabla \ell(\,\cdot\,,z)\| \leq L$ for all $z \in \mathcal{Z}$, and get
\begin{align*}
	\eqref{equ:proof:smthstability:equ3} 
\le \Big( \delta(w,k)^{t-1}_w + 2 \eta L \Big)\frac{P_{vw}}{m}.
\end{align*}
To bound terms \eqref{equ:proof:smoothstab:equ1} and \eqref{equ:proof:smoothstab:equ2}, we use the non-expansive property of the gradient updates arising from the convexity and smoothness of $\ell(\,\cdot\,,z)$, specifically, the inequality  $\|x - y 
- \eta(\nabla \ell(x,z) - \nabla \ell(y,z))\| \leq \|x - y\|$ for $x,y \in \mathbb{R}^d, z \in \mathcal{Z}$  in Lemma \ref{lem:gradientcontract}. This yields
\begin{align*}
 \E[\delta(w,k)^{t}_v|\mathcal{F}_{t-1}] & \leq 
\sum_{l \not= w } P_{vl} \delta(w,k)^{t-1}_l + 
\Big( 1 - \frac{1}{m}\Big) P_{vw}  \delta(w,k)^{t-1}_w + 
\Big( \delta(w,k)^{t-1}_w + 2 \eta L \Big) \frac{P_{vw}}{m} \\
& = \sum_{l \in V} P_{vl} \delta(w,k)^{t-1}_l + 
\frac{2 \eta L}{m} P_{vw}.
\end{align*}
Let $e_v\in\R^n$ be the vector with $1$ in the coordinate aligning with node $v$ and $0$ everywhere else. In vector form, with $\delta(w,k)^t = \{ \delta(w,k)^t_v\}_{v \in V} \in\mathbb{R}^n$, the bound above takes the following form (the inequality is meant coordinate-wise)
\begin{align*}
\E\, \delta(w,k)^t  = \E [ \E[\delta(w,k)^t  |\mathcal{F}_{t-1}] ] 
& \leq P\, \E\, \delta(w,k)^{t-1} +  \frac{2 \eta L}{m} P e_w  
\leq \frac{2 \eta L}{m}\sum_{s=1}^{t-1}  P^s e_w,
\end{align*}
where we used $\delta(w,k)^{1}_l =\|\widetilde{X}(w,k)^1_l - X^1_l\| = 0$ for all $l \in V$. Recall $(P^s e_w)_v = P^s_{vw}$.
\end{proof}

The bound in Proposition \ref{prop:distributedstability} shows that the expected deviation between the algorithms remains zero until the number of iterations exceeds the natural distance in the graph between node $v$ and node $w$. This bound naturally reflects the graph topology and captures the propagation of the deviation due to resampling a data point in a specific location of the graph. When combined with the summation over $w \in V$ in Proposition \ref{prop:genfromstab}, this bound yields a Generalisation Error bound that does not depend on the graph topology: Lemma \ref{lem:GenSmoothDist}.\\

\begin{proof}[Lemma \ref{lem:GenSmoothDist}]
Plugging the bound from Proposition \ref{prop:distributedstability} into the identity from Proposition \ref{prop:genfromstab}, using that the rows of the matrix $P$ sum to $1$, we get
\begin{align*}
	\E [ r(X^t_v)-R(X^{t}_v)]
	\le \frac{L}{nm} \sum_{w\in V}\sum_{k=1}^m \E\, \delta(w,k)^t_v
	\le \frac{2 \eta L^2}{nm} 
	 \sum_{s=1}^{t-1}
	 \sum_{w \in V}  P^s_{vw} = \frac{2 \eta L^2}{nm} (t-1).
\end{align*}
\end{proof}

\subsubsection{Strongly Convex, Lipschitz, and Smooth Losses}
\label{sec:GenProofStrongly}
This section presents a Generalisation Error bound for Distributed SGD when the loss function is strongly convex, smooth, and Lipschitz continuous, generalising the results in \cite{Hardt:2016:TFG:3045390.3045520} to the distributed setting. Recall that a differentiable function $f:\R^d\rightarrow\R$ is $\gamma$-strongly convex, with $\gamma >0$, if $f(x) - f(y) \ge \nabla f(y)^\top(x - y) + \gamma \| x - y \|^2 /2$ for all $x,y\in\R^d$. As strongly convex functions have unbounded gradients on $\mathbb{R}^d$, we consider the setting where parameters are constrained to be on a compact convex set $\mathcal{X} \subset \mathbb{R}^d$. Let $x\rightarrow\Pi(x) = \arg \min_{y \in \mathcal{X}}\|x - y\|$ be the Euclidean projection on $\mathcal{X}$. Then, iteration \eqref{alg:ditributedSGD} becomes, for $s \geq 1$,
\begin{align}
\label{alg:DistributedProjSGD}
	X^{s+1}_v = \Pi
	\Big(
	\sum_{w\in V} P_{vw} ( X^{s}_w - \eta \nabla \ell(X^{s}_w,Z_{w,K^{s+1}_w}) \Big).
\end{align}
We refer to this variant as Distributed Projected SGD.

To motivate these assumptions, consider the specific case of Tikhonov regularisation, as done in \cite{Hardt:2016:TFG:3045390.3045520}. If the loss function $\ell$ is convex, $\beta$-smooth, and $L$-Lipschitz, then the penalised loss function $x\rightarrow \ell(x,z) + \frac{\gamma}{2} \|x\|^2$ is $\gamma$-strongly convex, ($\beta + \gamma$)-smooth, and $(L+\gamma r)$-Lipschitz when the constraint set is contained in a ball of radius $r$, i.e., $\mathcal{X}  \subseteq \{x \in \mathbb{R}^d: \|x\| \leq r \}$.
The next result is the analogue of Lemma \ref{lem:GenSmoothDist} with the additional assumption of strong convexity.

\begin{lemma}[Generalisation Error bound for strongly-convex, Lipschitz, and smooth losses]
\label{lem:GenStronglySmoothDist}$ $\\
Assume that for any $z\in\mathcal{Z}$ the function $\ell(\,\cdot\,,z)$ is $\gamma$-strongly convex, $L$-Lipschitz, and $\beta$-smooth. Let $X_v^1 = 0$ for all $v \in V$. Then, Distributed Projected SGD run on a compact, convex set $\mathcal{X}$ with  $\eta \le 2 / (\beta + \gamma)$ yields, for any $v\in V$ and $t \geq 1$,
\begin{align*}
\E [r(X_v^t)-R(X_v^{t})]
	\le 
	\frac{2 L^2 }{mn } \frac{\beta + \gamma}{\beta \gamma} .
\end{align*}
\end{lemma}
Observe that, for a sufficiently small step size  $\eta \leq 2/(\beta+\gamma)$, the bound obtained is independent of the step size $\eta$ and number of iterations $t$. As for the convex and smooth case of Lemma \ref{lem:GenSmoothDist}, also this bound aligns with the one given in \cite{Hardt:2016:TFG:3045390.3045520} for a single agent with $nm$ observations.

The next result is the analogue of Proposition \ref{prop:distributedstability}.

\begin{proposition}[Stability for strongly-convex, Lipschitz, and smooth losses]
\label{prop:distributedstabilityStrongly}
Assume the setting of Lemma \ref{lem:GenStronglySmoothDist}.
Then, for any $v,w\in V,k\in [m]$ and $t \geq 1$,
$$
	\E\, \delta(w,k)^t_v
	= \E \, \| \widetilde X(w,k)^t_v - X^t_v \| 
	\le  \frac{2 \eta L }{m} \sum_{s=1}^{t-1} 
	\Big( 1 - \frac{\eta \beta \gamma}{\beta + \gamma}\Big)^{s-1}  P^s_{vw}.
$$
\end{proposition}
\begin{proof}[Proposition \ref{prop:distributedstabilityStrongly}]
The proof follows the same outline for the proof of Proposition \ref{prop:distributedstability}. Consider the same setup and notation there defined. Using the non-expansive property of the Euclidean projection, the triangle inequality, and the fact that $\{X^{t-1}_v\}_{v\in V}, \{\widetilde X(w,k)^{t-1}_v\}_{v\in V}$, $\mathcal{D}$, and $\mathcal{\widetilde D}$ are measurable with respect to $\mathcal{F}_{t-1}$, we get
\begin{align}
& \E[\delta(w,k)^{t}_v|\mathcal{F}_{t-1}] \le \E \|  \widetilde X(w,k)^t_v - X^t_v\| \nonumber \\
& \leq
\sum_{l \not= w} P_{vl}
\E \Big[ 
\Big\|
\widetilde X(w,k)^{t-1}_l \!-\! X^{t-1}_l 
\!-\! \eta \Big(
\nabla \ell(\widetilde X(w,k)^{t-1}_l,Z_{l,K^{t}_l}) \!-\!  
\nabla \ell(X^{t-1}_l,Z_{l,K^{t}_l})
\Big)
\Big \| \Big| \mathcal{F}_{t-1} \Big]\label{equ:proof:strongstability:equ1}   \\
&\quad +
\frac{P_{vw}}{m} \sum_{i \not= k}\Big\|
\widetilde X(w,k)^{t-1}_w - X^{t-1}_w 
- \eta \Big(
\nabla \ell(\widetilde X(w,k)^{t-1}_w,Z_{w,i}) -  
	\nabla \ell(X^{t-1}_w,Z_{w,i})
\Big)
\Big\| \label{equ:proof:strongstability:equ2} \\
&\quad + \frac{P_{vw}}{m} \Big\|
\widetilde X(w,k)^{t-1}_w - X^{t-1}_w 
- \eta \Big(
\nabla \ell(\widetilde X(w,k)^{t-1}_w,\widetilde{Z}_{w,k}) -  
	\nabla \ell(X^{t-1}_w,Z_{w,k})
\Big)
\Big\| \label{equ:proof:strongstability:equ3}.
\end{align}
Term \eqref{equ:proof:strongstability:equ3} is the only one to involve the difference of two gradients evaluated at different data points ($\widetilde{Z}_{w,k}$ and $Z_{w,k}$). To use the contraction property arising from strong convexity, add and subtract the term $\eta \nabla \ell(\widetilde X(w,k)^{t-1}_w,Z_{w,k})$ inside the norm, and use the Lipschitz property
to get 
\begin{align*}
\eqref{equ:proof:strongstability:equ3} 
\leq 
\frac{P_{vw}}{m} \Big\|
\widetilde X(w,k)^{t-1}_w - X^{t-1}_w 
- \eta \Big(
\nabla \ell(\widetilde X(w,k)^{t-1}_w,Z_{w,k}) -  
	\nabla \ell(X^{t-1}_w,Z_{w,k})
\Big)
\Big\| + \frac{2 \eta L }{m}P_{vw}.
\end{align*}
To bound terms \eqref{equ:proof:strongstability:equ1} and \eqref{equ:proof:strongstability:equ2}, as well as the bound above for \eqref{equ:proof:strongstability:equ3}, we use the contraction property of the gradient updates from Lemma \ref{lem:gradientcontract}, specifically, the inequality $
\| x - y - \eta (\nabla \ell(x,z) - \nabla \ell(y,z))\| \leq 
( 1 - \frac{\eta \beta \gamma}{\beta + \gamma}) \|x - y \|
$ for $x,y \in \mathbb{R}^d$, $z \in \mathcal{Z}$, and $\eta \le  \frac{2}{\beta + \gamma}$. We get
\begin{align*}
	& \E [ \delta(w,k)^{t}_v | \mathcal{F}_{t-1} ] \\
	& \leq 
	\Big( 1 - \frac{\eta \beta \gamma}{\beta + \gamma}\Big) \bigg[ 
	\sum_{l \not= w} P_{vl}\delta(w,k)^{t-1}_l 
	+  \Big( 1 - \frac{1}{m}\Big) 
	P_{vw} \delta(w,k)^{t-1}_w + 
	\frac{1}{m} P_{vw} 
	\delta(w,k)^{t-1}_w 
	\bigg] + \frac{2 \eta L}{m} P_{vw} \\
	& = \Big( 1 - \frac{\eta \beta \gamma}{\beta + \gamma}\Big) 
	\sum_{l \in V} P_{vl} \delta(w,k)^{t-1}_l
	+ \frac{2 \eta L}{m} P_{vw}.
\end{align*}
In vector notation, the above reads
\begin{align*}
\E\, \delta(w,k)^{t} & \leq  
\Big( 1 - \frac{\eta \beta \gamma}{\beta + \gamma}\Big) 
P\, \E\,\delta(w,k)^{t-1} + \frac{2 \eta L}{m} P e_w  
 \leq \frac{2 \eta L }{m} \sum_{s=1}^{t-1} 
\Big( 1 - \frac{\eta \beta \gamma}{\beta + \gamma}\Big)^{s-1}  P^s e_w 
\end{align*}
where we used $\delta(w,k)^1_l = \| \widetilde X(w,k)^1_l - X^1_l\| =0$ for all $l \in V$.
\end{proof}
With Proposition \ref{prop:distributedstabilityStrongly} in hand, we prove Lemma \ref{lem:GenStronglySmoothDist}.\\

\begin{proof}[Lemma \ref{lem:GenStronglySmoothDist}]
Plugging the bound from Proposition \ref{prop:distributedstabilityStrongly} into the identity from Proposition \ref{prop:genfromstab}, using that the rows of the matrix $P$ sum to $1$, we get
\begin{align*}
	\E [r(X_v^t)-R(X_v^{t})]
	& \le 
	\frac{L}{nm}\sum_{w \in V}  \sum_{k=1}^m
	\E
	\,\delta(w,k)^t_v
	 \le \frac{2 \eta L^2}{mn} \sum_{s=1}^{t-1}  
	 \Big( 1 - \frac{\eta \beta \gamma}{\beta + \gamma}\Big)^{s-1},
\end{align*}
and the proof is concluded by summing the geometric projection for $t$ going to infinity, using that the assumption $\eta \le \frac{2}{\beta + \gamma}$ implies that $\frac{\eta \beta \gamma}{\beta + \gamma} < 1$.
\end{proof}

\subsection{Proof of Generalisation Error Bound for Non-Smooth Losses}
\label{app:Lipschitz:Gen}
This section presents Generalisation Error bounds for Distributed SGD when losses are assumed to be non-smooth, aligning with Lemma \ref{lem:Lipschitz:GenError} within the main body of the text.
In this case we follow the approach in \cite[Appendix B]{lin2016generalization} that involves controlling the Generalisation Error by using standard Rademacher complexity's arguments through Assumption \ref{assumptions} \textit{(b)} and  bounding the norm of the iterates through Assumption \ref{assumptions} \textit{(a)}.
We start by presenting Lemma \ref{lem:Lipschitz:IterBound} which bounds the iterates produced by the Distributed SGD. This is followed by the proof for the Generalisation Error bound for non-smooth losses Lemma \ref{lem:Lipschitz:GenError}.
\begin{lemma}
\label{lem:Lipschitz:IterBound}
Assume there exist $C \le B$ such that for each $z \in \mathcal{Z}$  the function $\ell(\,\cdot\,,z)$ is convex, $L$-Lipschitz, bounded above at zero $\ell(0,z) \leq B$, and bound uniformly from below $\ell(x,z) \ge C$ for $x \in \mathbb{R}^d$. Let $X^1_v = 0$ for all $v \in V$. Then, Distributed SGD yields, for any $v \in V$ and $t \geq 1$, 
$$
	\|X^t_v\| \leq \sqrt{(t-1)( \eta^2 L^2 + 2 \eta (B-C) )}.
$$ 
\end{lemma} 
\begin{proof}
Let $x \in \mathbb{R}^d$. By the Distributed SGD update \eqref{alg:ditributedSGD} we get
\begin{align}
\label{lem:Lipschitz:equ1}
	\|X^t_v - x \| \leq  
	\sum_{w \in V} P_{vw} 
	\|
	X^{t-1}_w - \eta \partial \ell(X^{t-1}_w,Z_{w,K^t_w}) - x
	\|.
\end{align}
The convexity of $\ell(\,\cdot\,,z)$ yields
$$	
	\langle \partial \ell(X^{t-1}_w,Z_{w,K^{t}_w}),x - X^{t-1}_w\rangle 
	\leq \ell(x,Z_{w,K^{t}_w}) - \ell(X^{t-1}_w,Z_{w,K^{t}_w}),
$$
and the Lipschitz continuity of $\ell(\,\cdot\,,z)$ yields $\|\partial \ell(X^{t-1}_w,Z_{w,K^t_w}) \| \leq L$. Thus,
\begin{align*}
& \|X^{t-1}_w - \eta \partial \ell(X^{t-1}_w,Z_{w,K^t_w})  - x \|^2  \\
& = 
\|X^{t-1}_w -  x\|^2 + \eta^2 \|\partial \ell(X^{t-1}_w,Z_{w,K^t_w}) \|^2  
+ 2 \eta \langle
\partial \ell(X^{t-1}_w,Z_{w,K^t_w}),x - X^{t-1}_w
\rangle \\
& \leq 
\|X^{t-1}_w -  x\|^2 + \eta^2 L^2 
+ 2\eta ( \ell(x,Z_{w,K^{t}_w}) - \ell(X^{t-1}_w,Z_{w,K^{t}_w}) ).
\end{align*}
Setting $x=0$, using that $\ell(X^{t-1}_w,Z_{w,K^{t}_w}) \geq C$ as well as the assumption $\ell(0,Z_{w,K^{t}_w}) \leq B$, we get
\begin{align*}
	\|X^{t-1}_w - \eta \partial \ell(X^{t-1}_w,Z_{w,K^t_w}) \|^2 
	 \leq 
	\|X^{t-1}_w \|^2 + \eta^2 L^2 
	+ 2\eta (B-C).
\end{align*}
Using that the matrix $P$ is doubly stochastic, the bound \eqref{lem:Lipschitz:equ1} yields the recursion
\begin{align*}
\max_{v\in V} \|X^t_v \|^2
\leq \max_{w \in V} \|X^{t-1}_w - \eta \partial \ell(X^{t-1}_w,Z_{w,K^t_w}) \|^2
\leq \max_{v \in V} \|X^{t-1}_v\|^2 + \eta^2 L^2 + 2 \eta (B-C),
\end{align*}
so that
$
	\|X^{t}_v\| \leq \max_{v \in V} \|X^t_v \|  \leq \sqrt{ (t-1) (\eta^2 L^2 + 
	2 \eta (B-C) )}.
$
\end{proof}


\begin{proof}[Lemma \ref{lem:Lipschitz:GenError}] Standard Rademacher complexity's arguments utilising the symmetrisation technique and Assumption \ref{assumptions} \textit{(b)} yield, for any $\widetilde{\mathcal{X}} \subseteq \mathcal{X}$,
$$
	\E \sup_{x \in \mathcal{\widetilde X}}(r(x) - R(x))
	\leq 
	2 
	\E \sup_{x \in \mathcal{\widetilde X}} 
	\frac{1}{nm} \sum_{i=1}^{nm} \sigma_i \ell(x,z_i)
	\leq 2 D \frac{\sup_{x \in \mathcal{\widetilde X}} \|x\| }{\sqrt{nm}}.
$$
By Lemma \ref{lem:Lipschitz:IterBound} we know that the iterates are contained in the ball $\mathcal{\widetilde X  } = \{x \in \mathbb{R}^d: \|x\| \leq \sqrt{(t-1)(\eta^2L^2 + 2\eta (B-C))} \}$, so that
\begin{align*}
	\E[ r(X^t_{v}) - R(X^t_v)] \leq 
	\E \sup_{x \in \mathcal{\widetilde X} }(r(x) - R(x))
	\leq 2 D \sqrt{ \frac{ (t-1)(\eta^2L^2 + 2\eta (B-C)) }{nm}}.
\end{align*}
\end{proof}

\subsection{Proof of Test Error Bounds for Smooth and Non-Smooth Losses}
\label{app:proof:TestError}
This section gives the proofs of the Test Error bounds presented within the main body of the work, namely Theorem \ref{thm:convexrisknmin} for convex, Lipschitz, and smooth losses, and Theorem \ref{thm:convexrisknmin_nonsmooth} for convex and Lipschitz losses. This is achieved by using the error decomposition given in Proposition \ref{prop:disterrorsplit}, and by bringing together the Generalisation Error bounds and the Optimisation Error bounds in Section \ref{sec:ErrorBounds}.\\

\begin{proof}[Theorem \ref{thm:convexrisknmin}]
By the convexity of the Test Risk $r$, using Proposition \ref{prop:disterrorsplit}, we get
\begin{align*}
	\E\,  r\Big( \frac{1}{t} \sum_{s=1}^{t} X_{v}^{s+1}\Big) - r(x^\star)  
	\leq
	\frac{1}{t} \sum_{s=1}^t 
	\Big(
	\underbrace{\E [r(X_{v}^{s+1}) - R(X_{v}^{s+1})]
	}_{\text{Generalisation Error}}
	+ 
	\underbrace{\E[R(X_{v}^{s+1}) - R(X^\star) ]
	}_{\text{Optimisation Error}}
	\Big).
\end{align*}
The proof follows by applying Lemma \ref{lem:GenSmoothDist} for the Generalisation Error, which yields
\begin{align*}
\frac{1}{t} \sum_{s=1}^t \E [r(X_{v}^{s+1}) - R(X_{v}^{s+1})] & \leq \frac{2 \eta L^2}{n m}  
\frac{1}{t} \sum_{s=1}^t s = \frac{ \eta L^2}{n m}(t+1),
\end{align*}
and by the Optimisation Error bound from Lemma \ref{lem:DistSGD:ConvexSmooth:OPT}.
\end{proof}
\begin{proof}[Theorem \ref{thm:convexrisknmin_nonsmooth}] 
By the convexity of the Test Risk $r$, using Proposition \ref{prop:disterrorsplit}, we get
\begin{align*}
& 	\E\,  r\Big( \frac{1}{t} \sum_{s=1}^{t} 
	X_{v}^{s}\Big) - r(x^\star)  
	\leq
	\frac{1}{t} \sum_{s=1}^t 
	\Big(\underbrace{ \E [r(X_{v}^{s}) - R(X_{v}^{{s}})]
	}_{\text{Generalisation Error}}
	+
	\underbrace{\E[R(X_{v}^{s}) - R(X^\star) ]
	}_{\text{Optimisation Error}}
	\Big).
\end{align*}
The proof follows by applying Lemma \ref{lem:Lipschitz:GenError} for the Generalisation Error, which yields
\begin{align*}
\frac{1}{t} \sum_{s=1}^{t} \E [r(X_{v}^{s}) - R(X_{v}^{{s}})]
& \leq 
2 D \sqrt{ \frac{ (t-1)(\eta^2L^2 + 2\eta (B-C)) }{nm}},
\end{align*}
and by the Optimisation Error bound from Lemma \ref{lem:DistSGD:Convex:OPT}.
\end{proof}

\subsection{Calculations for Corollary \ref{Cor:smoothrates} (Convex, Lipschitz, and Smooth)}
\label{App:Calculations:Smooth}
This section presents the calculations needed to populate the table of rates in Corollary \ref{Cor:smoothrates} in the case of convex, Lipschitz, and smooth losses. The simplification $1/(\beta + 1/\rho) \leq \rho$ is used. Additionally, minimisations are performed up to first-order terms in $\rho$, possibly neglecting logarithmic terms. This section is divided into four parts: 
\begin{itemize}
\item \textbf{Optimisation Error} calculates the step size $\rho^\star_{\mathrm{Opt}}$ minimising the Optimisation Error bound; 
\item \textbf{Test Error} calculates the step size $\rho^\star_{\mathrm{Test}}$ that minimises the Test Error bound; 
\item \textbf{Early Stopping Optimisation} calculates the number of iterations that minimises the Test Error bound when the step size 
$\rho^\star_{\mathrm{Opt}}$ is used; 
\item \textbf{Early Stopping Centralised} calculates the number of iterations that minimises the Test Error bound when the step size $\rho^\star=O(1/\sqrt{t})$ is used.
\end{itemize}
\paragraph{Optimisation Error.}
Optimising over first-order terms in  $\rho$ in the Optimisation Error bound of Lemma \ref{lem:DistSGD:ConvexSmooth:OPT} with $ 1/(\beta + 1/\rho)  \leq \rho $ we get
\begin{align*}
\rho_{\mathrm{Opt}}^{\star} = \argmin_{\rho} \Big\{ 
 	\frac{\rho}{2} \sigma^2 + \frac{ G^2}{2t \rho} +
 	3 L \kappa \rho  \frac{\log((t+1)\sqrt{n})}{1-\sigma_2(P)} 
 	\Big\}
= \frac{G}{\sqrt{t} } \frac{1}{\sqrt{ 6L \kappa \frac{ \log((t+1) \sqrt{n})}{ 1 - \sigma_2(P)}  + \sigma^2} },
\end{align*}
which yields the Optimisation Error bound 
\begin{align*}
	\E\Big[  
	R\Big( \frac{1}{t} \sum_{s=1}^{t} X^{s+1}_v \Big)	 
	-  R(X^\star) \Big]
	&  \leq \frac{G}{\sqrt{t}} \sqrt{ 6L \kappa \frac{ \log((t+1) \sqrt{n})}{ 1 - \sigma_2(P)}  
	+ \sigma^2}  \\
	&\quad+
	\frac{\beta G^2}{2 t} \Big[ 1 +
	\frac{9 (3+\beta \rho) \kappa }{6L }
	\frac{ \frac{\log^2((t+1)\sqrt{n})}{ (1-\sigma_2(P))^2 } }{
	\frac{ \log((t+1) \sqrt{n})}{ 1 - \sigma_2(P)}  + \frac{\sigma^2}{6L \kappa }  }
	\Big].
\end{align*}
This bound is $\widetilde{O}(1/\sqrt{ (1-\sigma_2(P))t})$ as the second term is 
$\widetilde{O}(1/ ((1-\sigma_2(P))t))$.
\paragraph{Test Error.} Consider the Test Error bound in Theorem \ref{thm:convexrisknmin} with $  1 /(\beta + 1/\rho)  \leq \rho$. Optimising over first-order terms in $\rho$ we get
\begin{align*}
\rho_{\mathrm{Test}}^\star & = \argmin_{\rho} \Big\{ 
\frac{\rho}{2} \sigma^2 + \frac{ G^2}{2t \rho} +
 	3 L \kappa \rho  \frac{\log((t+1)\sqrt{n})}{1-\sigma_2(P)} + \frac{\rho L^2 }{nm}(t+1) \Big\}\\
& = \frac{G}{\sqrt{t} } \frac{1}{\sqrt{ 6L \kappa \frac{ \log((t+1) \sqrt{n})}{ 1 - \sigma_2(P)}  + \sigma^2 + \frac{2L^2(t+1)}{nm} }},
\end{align*}
which yields the Test Error bound 
\begin{align*}
\E\,  r\Big( \frac{1}{t} \sum_{s=1}^{t} X^{s+1}_v \Big)
	-  r(x^\star) 
	& \leq \frac{G}{\sqrt{t}} 
	\sqrt{ 6L \kappa \frac{ \log((t+1) \sqrt{n})}{ 1 - \sigma_2(P)}  
	+ \sigma^2 + \frac{2 L^2}{nm}(t+1) } \\
	& \quad+ 
	\frac{\beta G^2}{2 t} \Big[ 1 +
	\frac{9 (3+\beta \rho) \kappa }{6L }
	\frac{ \frac{\log^2((t+1)\sqrt{n})}{ (1-\sigma_2(P))^2 } }{
	\frac{ \log((t+1) \sqrt{n})}{ 1 - \sigma_2(P)}  + \frac{1}{6L \kappa } (\sigma^2+ 
	\frac{2 L^2}{nm}(t+1))
	 }
	\Big].
\end{align*}
This bound is  
$\widetilde{O}\Big(\sqrt{\frac{1}{t(1-\sigma_2(P))} + \frac{1}{nm}}\Big)$ as the second term is $\widetilde{O}(1/((1-\sigma_2(P))t))$. 
This is $\widetilde{O}(\frac{1}{\sqrt{nm}})$ when $t \gtrsim n m / ( 1 - \sigma_2(P))$.
Note that $t\rightarrow (t+1)/t$ is decreasing and $(t+1)/t \leq 2$ for $t \geq 1$. 
\paragraph{Early Stopping Optimisation.}
Considering the Test Error bound from Theorem \ref{thm:convexrisknmin} with step size $\rho = \rho^\star_{\mathrm{Opt}}$ and $1/(\beta + 1/\rho) \leq \rho$ we get
\begin{align*}
&  \E\, r\Big( \frac{1}{t} \sum_{s=1}^{t} X^{s+1}_v \Big)-  r(x^\star) 
  \leq\frac{\beta G^2}{2 t} \Big[  1 + 
	\frac{9 (3+\beta \rho) \kappa }{6L }
	\frac{ \frac{\log^2((t+1)\sqrt{n})}{ (1-\sigma_2(P))^2 } }{ 
	\frac{ \log((t+1) \sqrt{n})}{ 1 - \sigma_2(P)}  + \frac{\sigma^2}{6L \kappa }  }
	\Big]\\	
	& \qquad\qquad\qquad +  G \Bigg[ 
	\frac{1}{\sqrt{t}} \sqrt{ 6L \kappa \frac{ \log((t+1) \sqrt{n})}{ 1 - \sigma_2(P)} + \sigma^2} + \frac{2 L^2\sqrt{t} }{nm  } 
\sqrt{ \frac{1}{ 6L \kappa \frac{ \log((t+1) \sqrt{n})}{ 1 - \sigma_2(P)}  + \sigma^2 }}
	\Bigg], 
\end{align*}
where $(t+1)/\sqrt{t} \leq 2 \sqrt{t}$ was used. The first term is $\widetilde{O}(1/(1-\sigma_2(P)t))$, while the second term 
 $O\Big(\sqrt{ \frac{\log(t \sqrt{n})}{ t ( 1 -\sigma_2(P)) }} + 
\frac{1}{nm} \sqrt{ \frac{t (  1- \sigma_2(P) ) }{\log(t \sqrt{n}) }}\Big)$ is dominant. To minimise the second term with respect to $t$, consider the more tractable form
\begin{align*}
& 
	\frac{1}{\sqrt{t}} \sqrt{ 6L \kappa 
	\frac{ \log((t+1) \sqrt{n})}{ 1 - \sigma_2(P)} + \sigma^2} + \frac{2 L^2\sqrt{t} }{nm  } 
\sqrt{ \frac{1}{ 6L \kappa \frac{ \log((t+1) \sqrt{n})}{ 1 - \sigma_2(P)}  + \sigma^2 }}
	\\ 
	&\qquad\qquad\qquad \leq
	\frac{\sigma }{\sqrt{t}} + 
	\sqrt{ 6 L \kappa \frac{ \log((t+1) \sqrt{n})}{ t(1 - \sigma_2(P))}}
	+ \frac{2 L^2  }{nm  } 
	\sqrt{ \frac{ t (1-\sigma_2(P))}{6 L \kappa \log((t+1)\sqrt{n})}}.
\end{align*}
An approximate minimiser in $t$ is given by
\begin{align*}
	 \frac{t}{\log((t+1)\sqrt{n})}  & =  \argmin_{c \geq 0} \Bigg\{
	 \sqrt{ \frac{ 6 L \kappa}{ c(1 - \sigma_2(P))}}
	+ \frac{2 L^2  }{nm  } 
	\sqrt{  \frac{  c(1-\sigma_2(P))}{6 L \kappa }}	
	 \Bigg\}
	  = 3 \frac{\kappa}{L} \frac{ nm  }{ 1 - \sigma_2(P) }.
\end{align*} 
This choice yields the Test Error bound
\begin{align*}
\E\,  r\Big( \frac{1}{t} \sum_{s=1}^{t} X^{s+1}_v \Big) -  r(x^\star) 
& \leq 
\frac{\beta G^2 L (1-\sigma_2(P)) }{6 \kappa nm } \Big[  1 + 
	\frac{9 (3+\beta \rho) \kappa }{6L }
	\frac{ \frac{\log^2((t+1)\sqrt{n})}{ (1-\sigma_2(P))^2 } }{ 
	\frac{ \log((t+1) \sqrt{n})}{ 1 - \sigma_2(P)}  + \frac{\sigma^2}{6L \kappa }  }
	\Big] \\
	&\quad + \frac{G}{\sqrt{nm}} 
	\Big[ \sigma \sqrt{ \frac{ L (1 - \sigma_2(P))  }{ 3 \kappa}}
	+ 
	2\sqrt{2} L \Big],
\end{align*}
which is a $\widetilde{O}(\frac{1}{\sqrt{nm}})$ Test Error bound obtained with $t \simeq nm/(1-\sigma_2(P))$ iterations.  
\paragraph{Early Stopping Centralised.} 
Considering the Test Error bound of Theorem \ref{thm:convexrisknmin} with $1/{\beta + 1/\rho} \leq \rho$  and $\rho = \rho^\star= \frac{G}{Lc \sqrt{ t }}$ we get
\begin{align*}
\E\,  r\Big( \frac{1}{t} \sum_{s=1}^{t} X^{s+1}_v \Big) -  r(x^\star)   
& \leq 
	\frac{G}{c} \Big[ 
	3 \kappa  \frac{\log((t+1)\sqrt{n})}{ (1-\sigma_2(P))\sqrt{t} }   
	+ \frac{2  L }{ nm} \sqrt{t} 
	\Big]  \\
	& \quad  + \frac{G }{2 \sqrt{t} }\Big( \frac{\sigma^2}{cL}  + cL \Big)
	+ \frac{ \beta G^2}{2t} \Big[1 +  
	 \frac{9 (3+ \beta \rho)\kappa^2 }{ c^2 L^2 } \frac{\log^2((t+1)\sqrt{n})}
	{( 1-\sigma_2(P))^2 }  \Big],
\end{align*}
where $(t+1)/\sqrt{t} \leq 2 \sqrt{t}$  for $t \geq 1$ was used on the Generalisation Error bound. The above bound is dominated by the first term which  is 
$\widetilde{O}\Big( \frac{1}{ ( 1 - \sigma_2(P)) \sqrt{t} } + \frac{\sqrt{t}}{nm} \Big)$.
Minimising up to log terms yields
$$
	t  =  \frac{3\kappa}{2L } \frac{nm }{1-\sigma_2(P)}.
$$
This choice yields the Test Error bound
\begin{align*}
& \E\,  r\Big( \frac{1}{t} \sum_{s=1}^{t} X^{s+1}_v \Big) -  r(x^\star)   \leq 
\frac{G}{c} \sqrt{ \frac{6 \kappa L}{nm(1-\sigma_2(P))}} \Big[ \log((t+1)\sqrt{n}) +1\Big] \\
& + \sqrt{ \frac{L (1-\sigma_2(P))}{ 6 \kappa nm} }\Big( \frac{\sigma^2}{cL}  + cL \Big)
 + \frac{ L \beta G^2 (1-\sigma_2(P))}{3 \kappa nm} \Big[1 +  
	 \frac{9  (3+ \beta \rho)\kappa^2 }{ c^2 L^2 } \frac{\log^2((t+1)\sqrt{n})}
	{( 1-\sigma_2(P))^2 }  \Big],
\end{align*}
which is dominated by the first term that is $\widetilde{O}\Big(\frac{1}{\sqrt{nm(1-\sigma_2(P))}} \Big)$, as the third term is $\widetilde{O}\Big( \frac{1}{nm(1-\sigma_2(P))}\Big)$.
\subsection{Calculations for Corollary \ref{Cor:lipschitzrates} (Convex and Lipschitz) }
\label{App:Calculations:Lipschitz}
This section presents the calculations needed to populate the table of rates in Corollary \ref{Cor:lipschitzrates} in the case of convex and Lipschitz losses. This section is divided into four parts: 
\begin{itemize}
\item \textbf{Optimisation Error} calculates the step size $\eta_{\mathrm{Opt}}^\star$ minimising the Optimisation Error bound; 
\item \textbf{Test Error} calculates the step size $\eta_{\mathrm{Test}}^\star$ that minimises the Test Error bound; 
\item \textbf{Early Stopping Optimisation} calculates the number of iterations that minimises the Test Error bound when the step size $\eta_{\mathrm{Opt}}^\star$ is used; 
\item \textbf{Early Stopping Centralised} calculates the number of iterations that minimises the Test Error when the step size $\eta^\star=O(1/\sqrt{t})$ is used.
\end{itemize}
\paragraph{Optimisation Error.} 
Minimising the Optimisation Error bound in Lemma \ref{lem:DistSGD:Convex:OPT}  with respect to the step size yields $\eta = \eta^{\star}_{\mathrm{Opt}} = \frac{G}{L \sqrt{19 t}} \sqrt{\frac{1-\sigma_2(P)}{\log(t \sqrt{n})}} $ and
$$
	\E\Big[ R \Big( \frac{1}{t} \sum_{s=1}^t X^s_v \Big) - R(X^\star) \Big]
	\le \sqrt{19} \frac{GL}{\sqrt{t}} \sqrt{\frac{\log(t \sqrt{n} ) }{ 1-\sigma_2(P) }} .
$$
\paragraph{Test Error.} In this section the step size 
$$\eta = \eta^{\star}_{\mathrm{Test}} = \frac{G}{L\sqrt{t}} 
\frac{1}{  \sqrt{ \frac{19}{2}  \frac{\log( t \sqrt{n})}{1-\sigma_2(P))} 
+ \frac{ t }{ (nm)^{2/3} } } }$$ 
is shown to ensure that the Test Error bound in Theorem \ref{thm:convexrisknmin_nonsmooth} converges in a time uniform manner to a quantity of order
$\widetilde{O}(1/(nm)^{1/3})$. We consider the Optimisation and Generalisation Error separately. The Optimisation Error bound with this step size yields
\begin{align*}
	 \frac{19 }{2} 
	\frac{\eta^{\star}_{\mathrm{Test}}  L^2 \log(t \sqrt{n}) }{1-\sigma_2(P) }
	+ \frac{G^2}{2 \eta^{\star}_{\mathrm{Test}} t} 
	& = 
	\frac{GL}{\sqrt{t}}  
	\sqrt{ \frac{19}{2}  \frac{\log(t \sqrt{n})}{1-\sigma_2(P))} 
	+ \frac{ t }{ (nm)^{2/3} } }
	\Big[	
	\frac{ \frac{19}{2}   \frac{ \log(t \sqrt{n})}{1 - \sigma_2(P)} }{
	\frac{19}{2}  \frac{ \log(t \sqrt{n}) }{ 1 - \sigma_2(P)}
	+ \frac{  t}{ (nm)^{2/3} }   }  
	 + \frac{1}{2}  
	\Big]\\
	& \leq 	\frac{3}{2} \frac{GL}{\sqrt{t}}  
	\sqrt{ \frac{19}{2}  \frac{\log(t \sqrt{n})}{1-\sigma_2(P))} 
	+ \frac{ t }{ (nm)^{2/3} } },
\end{align*}
which is  $\widetilde{O}(\frac{1}{(nm)^{1/3}})$ when the number of iterations satisfies 
$t \geq \frac{19}{2} \log(t \sqrt{n}) (nm)^{2/3}/(1-\sigma_2(P))$.
We split the Generalisation Error bound term into two parts
\begin{align}
\label{equ:cor:nonsmoothcalc:equ1}
	2 D \sqrt{ \frac{ (t-1) (\eta^2L^2 + 2 \eta (B-C) ) }{nm}}  \leq
	2\eta D L \sqrt{\frac{t}{nm}} + 2 \sqrt{2} D \sqrt{\frac{ \eta t (B-C) }{nm}}
\end{align}
and bounded each part separately. The first quantity in \eqref{equ:cor:nonsmoothcalc:equ1} with the step size $\eta = \eta^{\star}_{\mathrm{Test}}$ becomes
\begin{align*}
2 \eta^{\star}_{\mathrm{Test}} D L  \sqrt{  \frac{t}{nm}} 
= \frac{GD}{\sqrt{nm}} 
\frac{1}{  \sqrt{ \frac{19}{2}  \frac{\log(t \sqrt{n})}{1-\sigma_2(P))} 
+ \frac{ t }{ (nm)^{2/3} } } } \leq
\frac{GD}{\sqrt{nm}}  \sqrt{\frac{2}{19} \frac{1-\sigma_2(P)}{\log(t \sqrt{n})}},
\end{align*}
which is $O(1/\sqrt{nm})$, and thus $O(1/(nm)^{1/3})$. For the second quantity in \eqref{equ:cor:nonsmoothcalc:equ1}, its square yields
\begin{align*}
8 D^2 \frac{\eta^{\star}_{\mathrm{Test}} t  (B-C) }{nm} 
& = 8 D^2 \frac{(B-C) G   }{ Lnm } 
\sqrt{ \frac{t }{   \frac{19}{2}  \frac{\log(t \sqrt{n})}{1-\sigma_2(P))} 
+ \frac{t}{(nm)^{2/3} }  } }\\
& = 8 D^2  \frac{(B-C)GL  }{L (nm)^{2/3}} 
\sqrt{ \frac{t }{   \frac{19}{2}  \frac{\log(t \sqrt{n})(nm)^{2/3} }{1-\sigma_2(P))} 
+  t   } } \\
& \leq 8 D^2  \frac{(B-C)G }{L (nm)^{2/3}}.
\end{align*}
Therefore, when using step size $\eta_{\mathrm{Test}}^\star$ with $t \gtrsim (nm)^{2/3}/(1-\sigma_2(P))$ the Test Error is bounded by the sum of three quantities each of which are $\widetilde{O}(1/(nm)^{1/3})$.  

\paragraph{Early Stopping Optimisation.}
Setting $\eta = \eta^\star_{\mathrm{Opt}}$ in the Test Error bound in Theorem \ref{thm:convexrisknmin_nonsmooth} and using \eqref{equ:cor:nonsmoothcalc:equ1} to split the Generalisation Error we get
\begin{align*}
\E\,r \Big( \frac{1}{t} \sum_{s=1}^t X^s_v \Big) - r(x^\star)  
&\leq \sqrt{19} \frac{GL}{\sqrt{t}} \sqrt{\frac{\log(t  \sqrt{n} ) }{ 1-\sigma_2(P) }} \\
&\quad+ \frac{2GD}{\sqrt{19 nm}} 
\sqrt{\frac{ 1-\sigma_2(P)}{\log(t \sqrt{n})}} 
+  2 \sqrt{2} D \sqrt{\frac{G (B-C)}{ L nm} \sqrt{\frac{t (1 - \sigma_2(P))}{19 
\log(t \sqrt{n})}}}.
\end{align*}
This is $O\Big(
\sqrt{\frac{\log(t \sqrt{n})}{t (1-\sigma_2(P))}} + 
\sqrt{\frac{1}{nm} \sqrt{\frac{t (1-\sigma_2(P))}{\log(t \sqrt{n})}}}
\Big)$ as the second term is dominated by the first and third. Approximately minimising the above with respect to the number of iterations by setting 
$t = 19 \log(t \sqrt{n})) L^2 (Gnm)^{2/3} / ( (1-\sigma_2(P)) (2(B-C))^{2/3} D^{4/3} )$ yields 
\begin{align*}
& \E\,r \Big( \frac{1}{t} \sum_{s=1}^t X^s_v \Big) - r(x^\star)  
\leq
 2^{1/3} 3\frac{ G^{2/3} (B-C)^{1/3} D^{2/3}}{(nm)^{1/3}}  + 
 \frac{2GD}{\sqrt{19 nm}} 
\sqrt{\frac{ 1-\sigma_2(P)}{\log(t \sqrt{n})}}.
\end{align*}
This is a $O(1/(nm)^{1/3})$ Test Error bound obtained with $t \simeq (nm)^{2/3}/ (1-\sigma_2(P))$ iterations. 
\paragraph{Early Stopping Centralised.}
Setting $\eta = \eta^\star = \frac{G}{L \sqrt{19 t}}$ in the Test Error bound in Theorem \ref{thm:convexrisknmin_nonsmooth} and using \eqref{equ:cor:nonsmoothcalc:equ1} to split the Generalisation Error gives
\begin{align*}
\E\,r \Big( \frac{1}{t} \sum_{s=1}^t X^s_v \Big) - r(x^\star)  
& \leq 
GL \sqrt{ \frac{19}{ t }}
\frac{ \log(t \sqrt{n})}{1-\sigma_2(P)} +  \frac{2 G D}{\sqrt{19 nm}} + 
2 \sqrt{2} D \sqrt{\frac{ G (B-C)}{L nm } 
\sqrt{\frac{t}{19} }},
\end{align*}
which is $\widetilde{O}\Big(\frac{1}{(1-\sigma_2(P))\sqrt{t}}+ \sqrt{\frac{1}{nm}\sqrt{t}}\Big)$ as the second term is dominated by the first and third. Approximately minimising the above with respect to the number of iterations by setting 
$$t = 19 \log^{4/3}(t \sqrt{n})(Gnm)^{2/3} L^2 /((1-\sigma_2(P)^{4/3} (2(B-C))^{2/3}  D^{4/3} )$$ gives 
\begin{align*}
\E\,r \Big( \frac{1}{t} \sum_{s=1}^t X^s_v \Big) - r(x^\star)  
& \leq 2^{1/3} 3 \Big(\frac{\log(t \sqrt{n})}{1-\sigma_2(P)}\Big)^{1/3} 
\frac{G^{2/3}(B-C)^{1/3}D^{2/3}}{(nm)^{1/3}} + \frac{2 G D}{\sqrt{19 nm}}.
\end{align*}
This is $\widetilde{O}(1/(nm(1-\sigma_{2}(P))^{1/3})$ and is obtained with $t \simeq (nm)^{2/3}/(1-\sigma_2(P))^{4/3}$ iterations.

\section{Proofs of Optimisation Error bounds}
\label{sec:distOptimBounds}
This appendix presents Optimisation Error  bounds for the Distributed Stochastic Subgradient Descent algorithm. Here we consider the general setting of stochastic first-order oracles. The Optimisation Error bounds presented within the main body of this work, specifically Lemma \ref{lem:DistSGD:ConvexSmooth:OPT} and Lemma \ref{lem:DistSGD:Convex:OPT} for smooth and non-smooth losses, follow from Corollary \ref{cor:diststochopt_smooth} and Corollary  \ref{cor:diststochopt}  within this appendix.

\subsection{Setup}
\label{sec:setupOpt}
Let $(V,E)$ be a simple undirected graph with $n$ nodes, and let $P\in\R^{n\times n}$ be a doubly stochastic matrix supported on the graph, i.e., $P_{ij} \not= 0$ only if $\{i,j\} \in E$. For each $v\in V$, let $F_v:\R^d\rightarrow\R$ be a random convex function. We consider the problem of minimizing the function
$
	x\rightarrow \overline F(x):=\frac{1}{n}\sum_{v\in V} F_v(x)
$
over $x\in\R^d$. Let $X^\star$ be a minimum of $\overline F$. Assume that $\E[\| X^\star \|^2] \le G^2$ for a positive constant $G$. Given the initial vectors $\{X^1_v=0\}_{v\in V}$, throughout this appendix, we consider the following update for $s \ge 1$:
\begin{align}
	X^{s+1}_v = \sum_{w\in V} P_{vw} ( X^{s}_w - \eta G_{w}^{s+1}),
	\label{alg:ditributedgrad}
\end{align}
where $\eta>0$ is the step size, and each $G_{v}^{s+1}\in\R^d$ is an estimator of the subgradient of $F_v$ evaluated at $X^{s}_v$. Specifically, for each $s\ge 1$ let $\mathcal{F}_{s}$ be the $\sigma$-algebra generated by the random functions $\{F_v\}_{v\in V}$ and by the estimators $\{G_v^k\}_{k \in \{2,\ldots,s\}}$. 
We have, for any $s\ge 1$, $v\in V$,
\begin{align}
	\E[G_{v}^{s+1}|\mathcal{F}_{s}] &\in \partial F_v(X^{s}_v). \label{AppendixB:Assumptions1}
\end{align}
Note that both $\{X^s_v\}_{v\in V}$ and $X^\star$ are measurable with respect to $\mathcal{F}_s$.
Assume, for any $s\ge 1$, $v\in V$,
\begin{align}
	\E[ \|G_{v}^{s+1}\|^2 |\mathcal{F}_{s}] &\leq L^2. \label{AppendixB:Assumptions2}
\end{align}
%
Section \ref{App:Opt:LipConvex} presents results for the setting just introduced under the additional assumption that the functions $\{F_v\}_{v\in V}$ are $L$-Lipschitz. Section \ref{app:Opt:Smooth} presents results for the case where the functions $\{F_v\}_{v\in V}$ are smooth (Lipschitz continuity is not assumed in this case). Finally, Section \ref{proof:theorem:ConvexOPT} checks that the assumptions of this general setting are satisfied for the specific case of algorithm \eqref{alg:ditributedSGD}.

The bounds that we derive are proved controlling the deviation of the output of the algorithm $X_v^s$ from its network average $\overline X^s:=\frac{1}{n}\sum_{v\in V} X^s_v$ on the one hand (\emph{network term}), and bounding the deviation of $\overline X^s$ from the solution $X^\star$ on the other end (\emph{optimisation term}). This strategy was originally proposed in \cite{Nedic2009} and used in \cite{DAW12} to get bounds that depend on the graph topology for a dual method in constrained optimisation. In the smooth case, we present a bound that also depends on the noise of the gradient (\emph{gradient noise term}).
\begin{remark}
\label{rem:comparisonToCentralised}
The bounds that we derive naturally generalise the bounds in the centralised setting. If no gradient noise is present and all the functions $\{F_v\}_{v\in V}$ are the same, then the network terms vanish as there is no difference between $X_v^s$ and $\overline X^s$ (recall that the initial conditions are the same for each node, i.e., $X^1_v=0$ for all $v\in V$) and optimal tuning of the step sizes recovers the same rates as for Centralised SGD: $O(1/\sqrt{t})$ for the Lipschitz case and $O(1/t)$ for the smooth case.
\end{remark}
As the matrix $P$ is doubly stochastic, the network average $\overline X^{s}$ admits the following simple evolution:
\begin{align}
	\overline X^{s+1}
	= \overline X^{s}
	- \eta \frac{1}{n} \sum_{v\in V} G_{v}^{s+1}.
	\label{alg:ditributedgradaverage}
\end{align}
In particular, note that by rearranging the previous expression we get
\begin{align}
	\frac{1}{n} \sum_{v\in V} G_{v}^{s+1}
	= \frac{1}{\eta}(\overline X^{s} - \overline X^{s+1}),
	\label{alg:ditributedgradaverageOrdered}
\end{align}
which will be used in the proofs in the next sections.


Before moving on to the next sections and presenting the Optimisation Error bounds, we establish bounds on the network terms that hold in the setting introduced so far. The next proposition bounds the deviation of $X^s_v$ from $\overline X^s$ as a function of the second largest eigenvalue in magnitude of the matrix $P$, i.e., $\sigma_2(P)$. We present different bounds, that either depend on the Lipschitz parameter $L$ or on a \emph{Gradient Noise and Function Deviation Term} $\kappa$, as defined in \eqref{def:networkNoise}. If no gradient noise is present and all the functions $\{F_v\}_{v\in V}$ are the same, then $\kappa=0$, reflecting the comment in Remark \ref{rem:comparisonToCentralised}.

\begin{proposition}[Network term]
\label{prop:network}
Consider the assumptions of Section \ref{sec:setupOpt}.
Let $\kappa^2$ be such that, for any $v\in V,s \geq 1$,
\begin{align}
        \underbrace{\E \Big[\Big\| G_{v}^{s+1} - \frac{1}{n}\sum_{\ell=1}^n \nabla F_\ell(X^{s}_\ell) \Big\|^2\Big]}_{\text{Gradient Noise and Function Deviation Term}}
         &\le \kappa^2.
        \label{def:networkNoise}
\end{align}
For any $v\in V,s \geq 1$, we have
\begin{align*}
        \E[\| X^{s}_v - \overline X^{s}\|^2]
        &\le \eta^2 \min\{L^2, \kappa^2\} \bigg( 2 \frac{\log (s\sqrt{n})}{1-\sigma_2(P)} + 1 \bigg)^2.
\end{align*}
\end{proposition}

\begin{proof}
Fix $v\in V,s\ge 1$. By unraveling the updates in \eqref{alg:ditributedgrad} and \eqref{alg:ditributedgradaverage}, using that $X^1_v=0$ for all $v\in V$, we get
\begin{align*}
	X^{s}_v =
	- \eta \sum_{k=1}^{s-1} \sum_{w\in V} P_{vw}^{k} G_{w}^{s-k + 1},
	\qquad
	\overline X^{s} = 
	- \eta \sum_{k=1}^{s-1} 
	\sum_{w\in V} ({\textstyle\frac{1}{n}}11^\top)_{vw} G_{w}^{s-k + 1},
\end{align*}
where $1\in\R^n$ is the all-one vector. Using that the rows of the matrix $P$ sum to one, note that for any collection of vectors $\{\nu^{k}\}_{k=1}^{s-1}$ in $\R^d$ we have
\begin{align*}
	X^{s}_v - \overline X^{s}
	&=
	\eta \sum_{k=1}^{s-1} \sum_{w\in V} 
	(P^{k}-{\textstyle\frac{1}{n}}11^\top)_{vw} (G_{w}^{s-k +1} - \nu^{s-k}).
\end{align*}
We have
\begin{align*}
	&\E[\| X^{s}_v - \overline X^{s}\|^2]
	= \E \langle X^{s}_v - \overline X^{s}, X^{s}_v - \overline X^{s} \rangle\\
	&\le
	\eta^2 \sum_{k,k'=1}^{s-1} \sum_{w,w'\in V} 
	|(P^{k}-{\textstyle\frac{1}{n}}11^\top)_{vw}|
	|(P^{k'}-{\textstyle\frac{1}{n}}11^\top)_{vw'}|
	\,\E |\langle G_{w}^{s-k +1} - \nu^{s-k}, G_{w'}^{s-k' +1} - \nu^{s-k'} \rangle|.
\end{align*}
By Cauchy-Schwarz's inequality and H\"older's inequality,
$$
	\E |\langle G_{w}^{s-k +1} - \nu^{s-k}, G_{w'}^{s-k' +1} - \nu^{s-k'} \rangle|
	\le \sqrt{\E[ \| G_{w}^{s-k +1} - \nu^{s-k}\|^2]} 
	\sqrt{\E[ \| G_{w'}^{s-k' +1} - \nu^{s-k'} \|^2]},
$$
and the above yields
\begin{align*}
	\E[\| X^{s}_v - \overline X^{s}\|^2]
	\le \bigg(
	\eta \sum_{k=1}^{s-1} \sum_{w\in V} 
	|(P^{k}-{\textstyle\frac{1}{n}}11^\top)_{vw}|
	\sqrt{\E[ \| G_{w}^{s-k +1} - \nu^{s-k}\|^2]}
	\bigg)^2.
\end{align*}
%
%
%
By choosing $\nu^{k} = 0$ and using \eqref{AppendixB:Assumptions2}, we get
\begin{align*}
	\E[\| X^{s}_v - \overline X^{s}\|^2]
	\le \eta^2L^2\bigg(
	\sum_{k=1}^{s-1} \sum_{w\in V} 
	|(P^{k}-{\textstyle\frac{1}{n}}11^\top)_{vw}|
	\bigg)^2.
\end{align*}
By choosing $\nu^{k}=\frac{1}{n}\sum_{\ell=1}^n \nabla F_\ell(X^{k}_\ell)$ and using the assumption of the proposition, we get
\begin{align*}
	\E[\| X^{s}_v - \overline X^{s}\|^2]
	\le \eta^2\kappa^2\bigg(
	\sum_{k=1}^{s-1} \sum_{w\in V} 
	|(P^{k}-{\textstyle\frac{1}{n}}11^\top)_{vw}|
	\bigg)^2.
\end{align*}
Note that
$
	\sum_{k=1}^{s-1} 
	\sum_{w\in V}
	|(P^{k}-{\textstyle\frac{1}{n}}11^\top)_{vw}|
	=
	\sum_{k=1}^{s-1} 
	\|e_v^\top P^{k}-{\textstyle\frac{1}{n}}1^\top\|_1,
$
where $e_v\in\R^n$ is the vector with $1$ in the coordinate aligning with node $v$ and $0$ everywhere else, and $\|\,\cdot\,\|_1$ denotes the $\ell_1$ norm. Standard results from Perron-Frobenius theory yield, for any $k\ge 1$,
$$
	\|e_v^\top P^{k}-{\textstyle\frac{1}{n}}1^\top\|_1
	\le \sqrt{n} \|e_v^\top P^{k}-{\textstyle\frac{1}{n}}1^\top\|
	\le \sqrt{n} \sigma_2(P)^k.
$$
To bound the quantity $\sum_{k=1}^{s-1} \|e_v^\top P^{k}-{\textstyle\frac{1}{n}}1^\top\|_1$,
we break the sum into two terms and bound them separately as follows:
\begin{align*}
	\sum_{k=1}^{s-1} 
	\|e_v^\top P^{k}-{\textstyle\frac{1}{n}}1^\top\|_1
	&=
	\sum_{k=1}^{\tilde s} 
	\|e_v^\top P^{k}-{\textstyle\frac{1}{n}}1^\top\|_1
	+
	\sum_{k=\tilde s +1}^{s-1} 
	\|e_v^\top P^{k}-{\textstyle\frac{1}{n}}1^\top\|_1
	\le
	2 \tilde s
	+
	\sqrt{n} \sum_{k=\tilde s +1}^{s-1}
	\sigma_2(P)^{k}.
\end{align*}
Requiring $\sigma_2(P)^{k} \le \frac{1}{s\sqrt{n}}$ for $k$ between $\tilde s +1$ and $s-1$, set $\tilde s = \lfloor \frac{\log (s\sqrt{n})}{\log (\sigma_2(P)^{-1})} \rfloor$. As there are no more than $s$ terms in the sum, using that $\log (x^{-1}) \ge 1- x$, we get
\begin{align*}
	\sum_{k=1}^{s-1} 
	\|e_v^\top P^{k}-{\textstyle\frac{1}{n}}1^\top\|_1
	&\le
	2 \tilde s + 1
	\le 2 \frac{\log (s\sqrt{n})}{1-\sigma_2(P)} + 1.
\end{align*}
\end{proof}

\subsection{Convex and Lipschitz} 
\label{App:Opt:LipConvex}

The following result controls the evolution of algorithm \eqref{alg:ditributedgrad} in the setting defined in Section \ref{sec:setupOpt}, under the additional assumption of Lipschitz continuity. The proof is inspired from the analysis in \cite{DAW12}, 

\begin{theorem}[Optimisation bound for convex and Lipschitz objectives]
\label{thm:diststochopt}
Consider the setting of Section \ref{sec:setupOpt}. Let the functions $\{F_v\}_{v\in V}$ be $L$-Lipschitz. Then, Distributed SGD yields, for any $v\in V$ and $t \geq 1$,
\begin{align*}
	\E\Big[
	\overline F\Big( \frac{1}{t}\sum_{s=1}^t X^s_v \Big)  
	- \overline F(X^\star) \Big]
	&\le \frac{1}{t} 
	\sum_{s=1}^t \E[ \overline F(X^s_v) - \overline F(X^\star)]\\
	&\le
	\underbrace{\frac{3 L}{t} \max_{w\in V} \sum_{s=1}^t \E  \| X^s_w \!-\! \overline X^s \|}_{\text{Network Term}}
	+ \underbrace{\frac{1}{t}\sum_{s=1}^t \frac{1}{n}\sum_{w\in V} 
	\E  \langle G_{w}^{s+1}, \overline X^s - X^\star \rangle}_{\text{Optimisation Term}}.
\end{align*}
and the Optimisation Term is upper bounded by 
$
	\frac{G^2}{2\eta t} + \frac{\eta L^2}{2}.
$
\end{theorem}
\begin{proof}
For any $s\ge 1$ and $v\in V$, adding and subtracting the term $\frac{1}{n}\sum_{w\in V} F_w(X^s_w)$, we find
\begin{align*}
	\E [ \overline F(X^s_v) - \overline F(X^\star) ]
	&= 
	\frac{1}{n}\sum_{w\in V} \E [F_w(X^s_v) - F_w(X^s_w)]
	+ \frac{1}{n}\sum_{w\in V} \E [ F_w(X^s_w) - F_w(X^\star)]
	\\
	&\le 
	\frac{1}{n}\sum_{w\in V} 
	L \E 
	\| X^s_v - X^s_w\|
	+ \frac{1}{n}\sum_{w\in V} 
	\E \langle G_{w}^{s+1}, X^s_w - X^\star \rangle,
\end{align*}
where for the first summand we used the Lipschitz property, and for the second summand we used convexity, assumption \eqref{AppendixB:Assumptions1}, and that both $\{X^s_v\}_{v\in V}$ and $X^\star$ are measurable with respect to $\mathcal{F}_s$. In fact, for any $w\in V$, we have
$$
	F_w(X^{s}_w) - F_w(X^\star)
	\!\le\! \langle \partial F_w(X^s_w), X^s_w - X^\star \rangle
	\!=\! \langle \E[G_{w}^{s+1}|\mathcal{F}_{s}], X^s_w - X^\star \rangle
	\!=\! \E[\langle G_{w}^{s+1}, X^s_w - X^\star \rangle|\mathcal{F}_{s}],
$$
so that
$ \E [F_w(X^{s}_w) - F_w(X^\star)]
\le \E\langle G_{w}^{s+1}, X^s_w - X^\star \rangle$
by the tower property of conditional expectations.
By adding and subtracting $\overline X^s$ and applying the Cauchy-Schwarz's inequality, we have  
\begin{align*}
	\E  \langle G_{w}^{s+1}, X^s_w - X^\star \rangle \leq 
	\E [\| G_{w}^{s+1} \| \|X^s_w  - \overline X^s \|] + 
	\E \langle G_{w}^{s+1},\overline X^s - X^\star \rangle,
\end{align*}
and the first term on the right-hand side is further bounded by using Jensen's inequality and the fact that $(X^s_w  - \overline X^s)$ is $\mathcal{F}_{s}$-measurable, along with assumption \eqref{AppendixB:Assumptions2}, giving 
\begin{align*}
	\E[ \| G_{w}^{s+1} \|  \|X^s_w  - \overline X^s \| ]
	\le \E[ (\E [  \| G_{w}^{s+1} \|^2 | \mathcal{F}_{s} ])^{1/2} \|X^s_w  - \overline X^s \| ]
	\leq L \E \|X^s_w  - \overline X^s \|.
\end{align*}
All together we have
\begin{align*}
	\frac{1}{t}\sum_{s=1}^t \E [ \overline F(X^s_v) - \overline F(X^\star)]
	& \le 
	\frac{3 L}{t} \max_{w\in V} \sum_{s=1}^t \E  \| X^s_w \!-\! \overline X^s \|
	+ \frac{1}{t} \sum_{s=1}^t \frac{1}{n}\sum_{w\in V} 
	\E \langle G_{w}^{s+1}, \overline X^s \!-\! X^\star \rangle.
\end{align*}
To bound the Optimisation Term we proceed as follows. Using \eqref{alg:ditributedgradaverageOrdered} and that
$\langle a,b \rangle = (\| a \|^2 + \| b \|^2 - \| a-b \|^2) / 2$ we obtain
\begin{align*}
	&\frac{1}{n}\sum_{w\in V} 
	\E  \langle G_{w}^{s+1}, \overline X^s - X^\star \rangle 
	= \frac{1}{\eta} \E
	\langle \overline X^{s} - \overline X^{s+1}, \overline X^s - X^\star \rangle\\
	&\qquad\qquad\qquad= \frac{1}{2\eta} 
	( \E[\|\overline X^{s+1} - \overline X^{s}\|^2] + \E [\| \overline X^s - X^\star \|^2] -
	 \E [\| \overline X^{s+1} - X^\star\|^2])\\
	&\qquad\qquad\qquad\le \frac{1}{2\eta} 
	\Big( \E [\| \overline X^s - X^\star \|^2] - \E [\| \overline X^{s+1} - X^\star\|^2] +
	\eta^2 \E \Big[\Big\| \frac{1}{n} \sum_{w\in V} G_{w}^{s+1} \Big\|^2\Big]\Big)\\
	&\qquad\qquad\qquad\le \frac{1}{2\eta} 
	( \E [\| \overline X^s - X^\star \|^2] - \E [\| \overline X^{s+1} - X^\star\|^2] +
	\eta^2 L^2),
\end{align*}
where we used Cauchy-Schwarz's and H\"older's inequalities, along with assumption \eqref{AppendixB:Assumptions2}, to get
\begin{align}
	\E \Big[\Big \| \frac{1}{n} \sum_{w\in V} G_{w}^{s+1} \Big\|^2\Big]
	&= \frac{1}{n^2} \sum_{u,w\in V} \E  \langle G_{u}^{s+1},G_{w}^{s+1}\rangle
	\leq 
	\frac{1}{n^2} \sum_{u,w \in V} 
	\E  [ \| G_{u}^{s+1}\| \| G_{w}^{s+1}\| ]\nonumber\\
	&\leq 
	\frac{1}{n^2} \sum_{u,w \in V} 
	\sqrt{\E  [ \| G_{u}^{s+1}\|^2]} \sqrt{\E [\| G_{w}^{s+1}\|^2 ]} \leq L^2.
	\label{equ:NormGradBound}
\end{align}
Summing over $s$, using that $X^1_v=0$ for all $v\in V$ and that $\E[\| X^\star \|^2] \le G^2$, we get the following bound for the Optimisation Term
\begin{align*}
	\frac{1}{t}\sum_{s=1}^t \frac{1}{n}\sum_{w\in V}  
	\E \langle G_{w}^{s+1}, \overline X^s - X^\star \rangle
	&\le \frac{1}{2\eta t} \E [\| \overline X^{1} - X^\star\|^2] + \frac{\eta L^2}{2}
	\le \frac{G^2}{2\eta t} + \frac{\eta L^2}{2}.
\end{align*}
\end{proof}


\begin{corollary}
\label{cor:diststochopt}
Consider the assumptions of Section \ref{sec:setupOpt}. Let the functions $\{F_v\}_{v\in V}$ be $L$-Lipschitz. Then, Distributed SGD yields, for any $v\in V$ and $t \geq 1$,
\begin{align*}
	\E\Big[
	\overline F\Big( \frac{1}{t}\sum_{s=1}^t X^s_v \Big)  
	- \overline F(X^\star) \Big]
	\le \frac{1}{t} 
	\sum_{s=1}^t \E[ \overline F(X^s_v) - \overline F(X^\star)]
	\le
	\frac{\eta L^2}{2} \Big( 19\frac{\log (t\sqrt{n})}{1-\sigma_2(P)} \Big)
	+ \frac{G^2}{2\eta t}.
\end{align*}
\end{corollary}

\begin{proof}
It follows from Theorem \ref{thm:diststochopt} and Proposition \ref{prop:network}, as $\E\| X^{s}_v - \overline X^{s}\|\le\sqrt{\E[\| X^{s}_v - \overline X^{s}\|^2]}$ by Jensen's inequality.
\end{proof}

\subsection{Convex
 and Smooth} 
\label{app:Opt:Smooth}
The following result controls the evolution of algorithm \eqref{alg:ditributedgrad} in the setting defined in Section \ref{sec:setupOpt}, under the additional assumption of smoothness. The proof is inspired by the one given \cite{dekel2012optimal} for centralised SGD applied to smooth losses, the specific exposition of which more closely follows  \cite{bubeck2015convex}. The bound that we give is made of three components: the Optimisation Term that decays like $1/t$; the Gradient Noise Term that captures the average noise of the gradient across the graph; the Network Term that controls the deviation of the algorithm from its network average.

\begin{theorem}[Optimisation bound for convex and smooth objectives]
\label{thm:diststochopt_smooth}
Consider the Assumptions of Section \ref{sec:setupOpt}.
Let the functions $\{F_v\}_{v\in V}$ be $\beta$-smooth.
Then, Distributed SGD with $\eta = 1/(\beta + 1/\rho)$ and $\rho \geq 0$, yields, for any $v \in V$ and $t \geq 1$,
\begin{align*}
	&\E\Big[\overline F\Big( \frac{1}{t} \sum_{s=1}^{t} X^{s+1}_v \Big) 
	- \overline F(X^\star) \Big] 
	\leq 
	\frac{1}{t} \sum_{s=1}^{t}  \E[\overline F( X^{s+1}_v )
	- \overline F(X^\star) ] \\
	&\le 
	\underbrace{
	\frac{1}{t} \!\sum_{s=1}^{t} \!\Big(L \E\|X^{s+1}_v \!-\! \overline X^{s+1}\|
	\!+\! \beta \max_{w\in V} \E[\| X^{s+1}_w \!-\! \overline X^{s+1} \|^2]
	\!+\! \frac{\beta}{2}\Big(1\!+\!\beta\rho\Big) \max_{w\in V} \E[\| X_w^{s} \!-\! \overline X^{s} \|^2]}_{\text{Network Term}} \Big)\\
	&\quad\!+\underbrace{\frac{\rho}{2}\frac{1}{t} \sum_{s=1}^{t} \E \Big[\Big\| \frac{1}{n} \sum_{w\in V}(G^{s+1}_w\!-\!\nabla F_w(X^{s}_w)) \Big\|^2\Big]}_{\text{Gradient Noise Term}}\\
	&\quad\!+\underbrace{\frac{1}{t}\sum_{s=1}^t \bigg(\frac{1}{n}\sum_{w\in V} \E \langle G^{s+1}_w, \overline X^{s+1} - X^\star \rangle + \frac{1}{2}\Big(\beta + \frac{1}{\rho}\Big) \E[\| \overline X^{s+1} - \overline X^{s} \|^2] \bigg) }_{\text{Optimisation Term}},
\end{align*}
and the Optimisation Term is upper bounded by 
$\frac{1}{2}(\beta+\frac{1}{\rho}) \frac{G^2}{t}$.
\end{theorem}

\begin{proof}
Recall that if a function $f:\R^d\rightarrow\R$ is $\beta$-smooth then for any $x,y \in \mathbb{R}^d$  we have \citep{nesterov2013introductory}
$f(x) - f(y) \leq  \langle \nabla f(y), x-y \rangle + 
\frac{\beta}{2} \| x - y\|^2 $. Fix $s\ge 1$, $v\in V$. Consider the following decomposition.
\begin{align}
	\overline F(X^{s+1}_v) - \overline F(X^\star)
	=
	\underbrace{\overline F(X^{s+1}_v) - \overline F(\overline X^{s+1})}_{\text{Term (a)}}
	+ \underbrace{\overline F(\overline X^{s+1}) - \overline F(X^\star)}_{\text{Term (b)}}.
	\label{TermaTermb}
\end{align}
\textbf{Term (a).} To bound Term (a), we use smoothness and convexity to get
\begin{align*}
	&\overline F(X^{s+1}_v) - \overline F(\overline X^{s+1})
	=
	\frac{1}{n} \sum_{w\in V} \Big(F_w(X^{s+1}_v) - F_w(X^{s+1}_w) 
	+ F_w(X^{s+1}_w) - F_w(\overline X^{s+1}) \Big)\\
	&\le \frac{1}{n} \sum_{w\in V} \Big( \langle \nabla F_w(X^{s+1}_w), X^{s+1}_v \!-\! X^{s+1}_w \rangle \!+\! \frac{\beta}{2} \| X^{s+1}_v \!-\! X^{s+1}_w \|^2
	+ \langle \nabla F_w(X^{s+1}_w), X^{s+1}_w - \overline X^{s+1} \rangle \Big)\\
	&= \frac{1}{n} \sum_{w\in V} \Big(\langle \nabla F_w(X^{s+1}_w), X^{s+1}_v - \overline X^{s+1} \rangle + \frac{\beta}{2} \| X^{s+1}_v - X^{s+1}_w \|^2\Big).
\end{align*}
As $\nabla F_w(X^{s+1}_w) = \E [G_w^{s+2}|\mathcal{F}_{s+1}]$ and $\{X^{s+1}_w\}_{w\in V}$ are $\mathcal{F}_{s+1}$-measurable, we get
\begin{align*}
	\langle \nabla F_w(X^{s+1}_w), X^{s+1}_v - \overline X^{s+1} \rangle
	&= \E [\langle G_w^{s+2}, X^{s+1}_v - \overline X^{s+1} \rangle |\mathcal{F}_{s+1}]\\
	&\le \E [\| G_w^{s+2}\| \| X^{s+1}_v - \overline X^{s+1} \| |\mathcal{F}_{s+1}]\\
	&\le \sqrt{\E [\| G_w^{s+2}\|^2  |\mathcal{F}_{s+1}]} \| X^{s+1}_v - \overline X^{s+1} \|\\
	&\le L \| X^{s+1}_v - \overline X^{s+1} \|,
\end{align*}
where we used Cauchy-Schwarz's inequality, Jensen's inequality, and $\E [\| G_w^{s+2}\|^2  |\mathcal{F}_{s+1}]\le L^2$.
Thus,
\begin{align}	
	\E [\overline F(X^{s+1}_v) - \overline F(\overline X^{s+1})]
	&\le L \E \|X^{s+1}_v - \overline X^{s+1}\|
	+ \beta \max_{w\in V} \E [\| X^{s+1}_w - \overline X^{s+1} \|^2].
	\label{smoothTerma}
\end{align}
\textbf{Term (b).} To bound Term (b), we use smoothness to find a bound that involves a telescoping sum whose terms cancel out when we take the summation over time $s$. Using smoothness and the Cauchy-Schwarz's inequality ($2\langle a,b\rangle \le \rho \|a \|^2 + \|b \|^2/\rho$ for $\rho\ge 0$) we get
\begin{align}
	\overline F(\overline X^{s+1}) - \overline F(\overline X^{s})
	&\le \frac{1}{n} \sum_{w\in V} \langle \nabla F_w(\overline X^{s}), \overline X^{s+1} - \overline X^{s} \rangle + \frac{\beta}{2} \| \overline X^{s+1} - \overline X^{s} \|^2\nonumber\\
	&= \Big\langle \frac{1}{n} \sum_{w\in V} (\nabla F_w(\overline X^{s}) - G^{s+1}_w), \overline X^{s+1} - \overline X^{s} \Big\rangle 
	+ \frac{1}{n} \sum_{w\in V} \langle G^{s+1}_w, \overline X^{s+1} - X^\star \rangle \nonumber\\
	&\quad\,+ \frac{1}{n} \sum_{w\in V} \langle G^{s+1}_w, X^\star - \overline X^{s} \rangle + \frac{\beta}{2} \| \overline X^{s+1} - \overline X^{s} \|^2\nonumber\\
	&\le \frac{\rho}{2} \Big\| \frac{1}{n} \sum_{w\in V} (\nabla F_w(\overline X^{s}) - G^{s+1}_w) \Big\|^2 
	+ \frac{1}{n} \sum_{w\in V} \langle G^{s+1}_w, \overline X^{s+1} - X^\star \rangle \nonumber\\
	&\quad\, + \frac{1}{n} \sum_{w\in V} \langle G^{s+1}_w, X^\star - \overline X^{s} \rangle + \frac{1}{2}\Big(\beta+\frac{1}{\rho}\Big) \| \overline X^{s+1} - \overline X^{s} \|^2.
	\label{GDsmoothness0int}
\end{align}
Adding and subtracting $\frac{1}{n}\sum_{w\in V} \nabla F_w(X^s_w)$ inside the inner product, using that $\{X^s_w\}_{w\in V}$ and $X^\star$ are $\mathcal{F}_s$-measurable, and that $\E[\nabla F_w(X^s_w)-G^{s+1}_w | \mathcal{F}_s] = 0$, we get
\begin{align}
	\E [\overline F(\overline X^{s+1}) - \overline F(X^\star)]
	&\le \E [\overline F(\overline X^{s}) - \overline F(X^\star)]
	+ \frac{\rho}{2} \E \Big[\Big\| \frac{1}{n} \sum_{w\in V} (\nabla F_w(\overline X^{s}) - G^{s+1}_w) \Big\|^2\Big]\nonumber\\
	&\quad\,+\frac{1}{n} \sum_{w\in V} \E \langle G^{s+1}_w, \overline X^{s+1} - X^\star \rangle
	+\frac{1}{2}\Big(\beta+\frac{1}{\rho}\Big) \E [\| \overline X^{s+1} - \overline X^{s} \|^2]
	\nonumber\\
	&\quad\,+ \frac{1}{n} \sum_{w\in V}  \E \langle \nabla F_w(X^{s}_w), X^\star - \overline X^{s} \rangle.
	\label{GDsmoothness0}
\end{align}
To bound the first term on the right-hand side of bound \eqref{GDsmoothness0} and cancel the dependence on $X^\star$ from the term $\langle\nabla F_w(X^{s}_w), X^\star - \overline X^{s} \rangle$, note that by convexity and smoothness we get
\begin{align}
	&\E [\overline F(\overline X^{s}) - \overline F(X^\star)]
	= \frac{1}{n}\sum_{w\in V} \E[ F_w(\overline X^{s}) - F_w(X^{s}_w)
	+ F_w(X^{s}_w) - F_w(X^\star)]\nonumber\\
	&\qquad\qquad= \frac{1}{n}\!\sum_{w\in V} \E[ F_w(\overline X^{s}) \!-\! F_w(X^{s}_w)
	\!+\! \langle \nabla F_w(X_w^{s}) ,X^s_w \!-\! \overline X^s \rangle
	\!+\! \langle \nabla F_w(X_w^{s}) ,\overline X^s \!-\! X^\star \rangle]\nonumber\\
	&\qquad\qquad\le \frac{\beta}{2} \max_{w\in V} \E \| X_w^{s} \!-\! \overline X^{s} \|^2 
	+ \frac{1}{n}\sum_{w\in V}\E\langle \nabla F_w(X_w^{s}) ,\overline X^s \!-\! X^\star\rangle.
	\label{GDsmoothness1}
\end{align}
To bound the second term on the right-hand side of bound \eqref{GDsmoothness0}, note that
\begin{align}
	&\E \Big[\Big\| \frac{1}{n}\! \sum_{w\in V} (\nabla F_w(\overline X^{s}) \!-\! G^{s+1}_w) \Big\|^2\Big]
	\!=\! \E \Big[\Big\| \frac{1}{n}\! \sum_{w\in V} \Big(\nabla F_w(\overline X^{s}) \!-\! \nabla F_w(X^{s}_w) \!+\! \nabla F_w(X^{s}_w) \!-\! G^{s+1}_w\Big) \Big\|^2\Big]\nonumber\\
	&\qquad= \E \Big[\Big\| \frac{1}{n} \sum_{w\in V} (\nabla F_w(\overline X^{s}) \!-\! \nabla F_w(X^{s}_w)) \Big\|^2\Big]
	+ \E \Big[\Big\| \frac{1}{n} \sum_{w\in V}(\nabla F_w(X^{s}_w) \!-\! G^{s+1}_w) \Big\|^2\Big],
	\label{GDsmoothness1a}
\end{align}
where we used that the cross terms are zero as $\E[G^{s+1}_w|\mathcal{F}_{s}] = \nabla F_{w}(X^s_w)$ and both $\{F_{w}\}_{w \in V}$ and $\{X^s_w\}_{w \in V}$ are $\mathcal{F}_s$-measurable. The first term in \eqref{GDsmoothness1a} can be bounded as follows:
\begin{align}
& \E \Big[ \Big\| \frac{1}{n} \sum_{w \in V} (\nabla F_{w}(X^{s}_w)  
- \nabla F_{w}(\overline{X}^s)) \Big\|^2  \Big] \nonumber\\
& =  \frac{1}{n^2} \sum_{w,l \in V} 
\E \langle 
\nabla F_{w}(X^{s}_w)   - \nabla F_{w}(\overline{X}^s),
\nabla F_{l}(X^{s}_l)   - \nabla F_{l}(\overline{X}^s) \rangle \nonumber\\
& \leq \frac{1}{n^2} \sum_{w,l \in V}  
\E \big[ 
\|\nabla F_{w}(X^{s}_w)   - \nabla F_{w}(\overline{X}^s)\|
\|\nabla F_{l}(X^{s}_l)   - \nabla F_{l}(\overline{X}^s)\| 
\big] \nonumber\\
 & \leq
 \frac{\beta^2 }{n^2} \sum_{w,l \in V}
 \E \big[ \| X^{s}_w   - \overline{X}^s\|
 \| X^{s}_l   - \overline{X}^s\| \big] \nonumber\\
 & \leq\frac{\beta^2}{n^2} \sum_{w,l \in V} 
  \sqrt{\E \big[  \| X_w^{s}   - \overline{X}^s\|^2\big]}
 \sqrt{\E \big[  \| X^{s}_l   - \overline{X}^s\|^2\big]}\nonumber\\
 &\leq \beta^2 \max_{w \in V} \E \big[  \| X^{s}_w   - \overline{X}^s\| ^2 \big],
 \label{GDsmoothness1b}
\end{align}
where applied Cauchy-Schwarz's inequality, smoothness, and H\"older's inequality.
Plugging \eqref{GDsmoothness1}, \eqref{GDsmoothness1a}, and \eqref{GDsmoothness1b} into \eqref{GDsmoothness0} we get the following bound for the expected value of term (b):
\begin{align}
	\E[\overline F(\overline X^{s+1}) \!-\! \overline F(X^\star)]
	&\le \frac{\beta}{2}\Big(1\!+\!\beta\rho\Big) \max_{w\in V} \E[\| X_w^{s} \!-\! \overline X^{s} \|^2]
	+\frac{\rho}{2}\E \Big[\Big\| \frac{1}{n} \sum_{w\in V}(\nabla F_w(X^{s}_w) \!-\! G^{s+1}_w) \Big\|^2\Big]\nonumber\\
	&\quad\, + \frac{1}{n}\sum_{w\in V} \E \langle G^{s+1}_w, \overline X^{s+1} - X^\star \rangle + \frac{1}{2}\Big(\beta + \frac{1}{\rho}\Big) \E [\| \overline X^{s+1} - \overline X^{s} \|^2].
	\label{smoothTermb}
\end{align}
\textbf{Term (a) + Term (b).} The main result in the theorem follows by using bounds \eqref{smoothTerma} and \eqref{smoothTermb} to bound Term (a) and Term (b) in \eqref{TermaTermb}, taking the summation over time from $s=1$ to $s=t$.

To bound the Optimisation Term, use \eqref{alg:ditributedgradaverageOrdered} and that $2\langle a,b\rangle =\|a\|^2 + \| b\|^2 - \| a - b \|^2$ so that
\begin{align*}
	\frac{1}{n}\sum_{w\in V} \langle G^{s+1}_w, \overline X^{s+1} - X^\star \rangle
	&= \frac{1}{\eta} \langle \overline X^s - \overline X^{s+1}, \overline X^{s+1} - X^\star \rangle\nonumber\\
	&= -\frac{1}{\eta} \langle \overline X^{s+1} - \overline X^{s}, \overline X^{s+1} - X^\star \rangle\nonumber\\
	&= \frac{1}{2\eta} \Big(-\| \overline X^{s+1} - \overline X^{s}\|^2 - \| \overline X^{s+1} - X^\star \|^2 + \| \overline X^{s} - X^\star \|^2 \Big).
\end{align*}
The choice $\eta=\frac{1}{\beta+1/\rho}$ leads to the cancellation of the quantity $\| \overline X^{s+1} - \overline X^{s} \|^2$ in the Optimisation Term. The telescoping sum over time, using that $X^1_w=0$ for all $w\in V$ and the assumption $\E[\| X^\star \|^2] \le G^2$, yields the final result.
\end{proof}

As for Centralised SGD \citep{dekel2012optimal}, the error bound that we give in Theorem \ref{thm:diststochopt_smooth} for the smooth case exhibits explicit dependence on the gradient noise, which in our setting is averaged out across the network. As far as the following corollary is concerned, we assume a time-uniform control on the gradient noise, namely,
\begin{align}
\label{AppendixB:Assumptions3}
\E \Big[\Big\| \frac{1}{n} \sum_{w\in V}(G^{s+1}_w\!-\!\nabla F_w(X^{s}_w)) \Big\|^2\Big] 
\leq 
\sigma^2
\end{align} 
for any $s\ge 1$.
\begin{corollary}
\label{cor:diststochopt_smooth}
Consider the Assumptions of Section \ref{sec:setupOpt}.
Let the functions $\{F_v\}_{v\in V}$ be $\beta$-smooth and satisfy 
both \eqref{def:networkNoise} and \eqref{AppendixB:Assumptions3}.
Then, Distributed SGD with $\eta = 1/(\beta + 1/\rho)$ and $\rho \geq 0$, yields, for any $v \in V$ and $t \geq 1$,
\begin{align*}
	& \E\Big[
	\overline F\Big( \frac{1}{t}\sum_{s=1}^t X^{s+1}_v \Big)  
	- \overline F(X^\star) \Big]
	\le \frac{1}{t} 
	\sum_{s=1}^t \E[ \overline F(X^{s+1}_v) - \overline F(X^\star)]\\
	& \le
	\frac{\rho}{2} \sigma^2 + \frac{(\beta + 1/\rho) G^2}{2t} + 
	\frac{3 \kappa}{\beta + 1/\rho} \frac{\log((t+1)\sqrt{n})}{1-\sigma_2(P)} 
	\Big( 
	L + 
	\frac{3}{2} \frac{ \beta(3+ \beta \rho)\kappa }{\beta + 1/\rho} \frac{\log((t+1)\sqrt{n})}{1-\sigma_2(P)} \Big) 
\end{align*}
\end{corollary}

\begin{proof}
It follows from Theorem \ref{thm:diststochopt_smooth} and Proposition \ref{prop:network}.
\end{proof}

\subsection{Assumptions for Distributed SGD \eqref{alg:ditributedSGD}}
\label{proof:theorem:ConvexOPT}
This section verifies that the more general assumptions considered in this Appendix for Distributed SGD \eqref{alg:ditributedgrad} are satisfied within the context of the main body of this work, that is, for Distributed SGD \eqref{alg:ditributedSGD} as described within Section \ref{sec:DistSGD}. This is performed by placing Distributed SGD \eqref{alg:ditributedSGD} into the context Distributed SGD \eqref{alg:ditributedgrad} as follows. Let the random objective functions be $F_v(x) = R_v(x) = \frac{1}{m} \sum_{k =1}^m \ell(x,Z_{v,k})$ for $v \in V$, which yields the network average $\overline F(x)  = R(x)$. Consider the following stochastic gradients, for $v \in V$ and $s \geq 1$, 
$$
	G_v^{s+1} = \partial \ell(X^s_v, Z_{v,K^{s+1}_v}),
$$
where $K^s_v$ is a uniform random variable on $[m]$.
Let $\mathcal{F}_1$ be the $\sigma$-algebra generated by the datasets $\mathcal{D}$. For any $s \geq 2$, let $\mathcal{F}_{s}$ contain the $\sigma$-algebra generated by the datasets $\mathcal{D}$ and the uniform random variables up to time $s$ $\{K^2_v,\dots,K^{s}_v\}_{v \in V} $. 
The random functions $\{F_{v}\}_{v \in V}$ and their optimal value $X^\star$ are $\mathcal{F}_s$-measurable, as $\mathcal{F}_{s}$ contains the $\sigma$-algebra generated by $\mathcal{D}$. The iterates $\{X_{v}^k\}_{ k \leq s, v \in V}$ are also $\mathcal{F}_{s}$-measurable, as $\mathcal{F}_{s}$ contains the $\sigma$-algebra generated by $\{K^2_v,\dots,K^{s}_v\}_{v \in V} $. 
We now check that assumption \eqref{AppendixB:Assumptions1} and assumption \eqref{AppendixB:Assumptions2} are satisfied. The following hold for any $s \geq 1$. 
\begin{itemize}
\item 
Assumption \eqref{AppendixB:Assumptions1} on the unbiasedness of the subgradient estimators is satisfied as for any $v \in V$ we have
\begin{align*}
	\E [ G_v^{s+1}|\mathcal{F}_s] & = 
	\E  [\partial \ell (X^{s}_v,Z_{v,K^{s+1}_v}) \big|\mathcal{F}_{s} ] 
	= \frac{1}{m} \sum_{k =1}^m \partial \ell(X^{s}_v,Z_{v,k}) \in \partial F_v(X^{s}_v),
\end{align*}
where have used that the sum of subgradients belong to the subgradient of sums.
\item
Assumption \eqref{AppendixB:Assumptions2} on the boundedness of the second moment of the subgradients is satisfied as for any $v\in V$ we have
\begin{align*}
	\E  [ \| G_v^{s+1}\|^2 | \mathcal{F}_s  ] & = 
	\E  [ \| \partial \ell 
	(X^{s}_v,Z_{v,K^{s+1}_v} ) \|^2 \big| \mathcal{F}_s  ] 
	= \frac{1}{m} \sum_{k=1}^m \| \partial \ell (X^{s}_v,Z_{v,k})\|^2 \leq L^2,
\end{align*}
where we have used that the function $\ell(\,\cdot\,,z)$ is $L$-Lipschitz for all $z \in Z$.
\end{itemize}

\bibliography{references}
\end{document}